\documentclass[twoside]{article}
\usepackage{amsmath,amsfonts,amssymb,amsthm}
\usepackage[preprint]{aistats2027}
\usepackage{float}
\usepackage{graphicx}
\usepackage{subcaption}
\usepackage{hyperref}
\usepackage{natbib}
\usepackage{cleveref}

\usepackage{comment}
\usepackage{booktabs}
\usepackage{enumitem}

\setlist[itemize]{wide=0pt} 

\hypersetup{
    colorlinks=true,
    linkcolor=blue,  
    citecolor=blue   
}
\theoremstyle{plain}
\newtheorem{theorem}{Theorem}[section]
\newtheorem{proposition}[theorem]{Proposition}
\newtheorem{lemma}[theorem]{Lemma}
\newtheorem{corollary}[theorem]{Corollary}
\theoremstyle{definition}
\newtheorem{definition}[theorem]{Definition}
\newtheorem{assumption}[theorem]{Assumption}
\theoremstyle{remark}

\crefname{assumption}{assumption}{assumptions}
\Crefname{assumption}{Assumption}{Assumptions}

\def\x{\mathbf{x}}
\def\L{\mathcal{L}}
\def\S{\mathcal{S}}
\def\R{\mathbb{R}}
\def\I{\mathbf{I}}
\def\w{\mathbf{w}}
\def\r{\mathbf{r}}
\def\brho{\boldsymbol{\rho}}
\def\H{\mathbf{H}}
\def\W{\mathbf{W}}
\def\h{\mathbf{h}}
\def\b{\mathbf{b}}
\begin{document}

%

%

\twocolumn[

\aistatstitle{On the Dynamics of Delayed Generalization in Linear Models}

\aistatsauthor{ Atreya Vedantam$^*$ \hspace{2em} Nataraj Das$^*$ \hspace{2em} Chandrasekhar Lakshminarayanan }

\aistatsaddress{ Indian Institute of Technology Madras } ]

\renewcommand{\thefootnote}{\fnsymbol{footnote}}
\footnotetext[1]{Equal contribution.}
\renewcommand{\thefootnote}{\arabic{footnote}}

\begin{abstract}
Delayed generalization, where training accuracy becomes nearly perfect substantially before test accuracy, has often been studied in the context of the depth and compositional structure of deep networks. In this paper, we ask whether delayed generalization can occur in simpler settings. We show that it can arise in a linear model trained with gradient descent and an exponential-tail loss on linearly separable data. We characterize the finite-time hyperplane dynamics into three phases: a population phase, a support vector dominated unlearning phase and a generalization phase. The convergence of the residual beyond the maximum-margin direction is sufficient for the unlearning phase to begin. In the support vector dominated phase, we characterize the hyperplane dynamics. The time scale in these dynamics causes delayed generalization. We validate our theory through controlled experiments varying population and support-vector distributions, and further verify the predicted bias behavior in a deep model with a single output neuron on MNIST. Our results provide a unified characterization of the asymptotic and finite-time dynamics underlying delayed generalization in linear models.
\end{abstract}

\section{Introduction}\label{sec:intro}

Modern machine learning systems generalize remarkably well from finite training data \cite{neyshabur, zhang, Papyan_2020, maa, kaplan}. Yet models often reach near-perfect training performance long before test performance improves \cite{Junior}. Despite growing evidence across architectures and tasks \cite{tan, fan, mohamadi}, a principled theory of this delayed generalization is still missing.

A natural starting point is binary classification on linearly separable data, trained by gradient descent with a constant step size $\eta$ on a smooth, decreasing loss with an exponential tail. In this setting, \citet{soudry} showed that the weights converge in direction to the max-margin SVM solution. They argued that this explains why optimizing the logistic or cross-entropy loss helps even after the training error reaches zero.

\textbf{Our contribution.} We give a finite-time analysis of delayed generalization in linear models. Our findings are summarized below.

\textbf{Question 1.} How does a linear model learn linearly separable data in finite time?

\textbf{Answer.} The weights converge in direction to the maximum-margin vector while their norm grows. The bias passes through three phases: \textit{Phase 1: population dominated} (\Cref{lem:popdyn}), \textit{Phase 2: support-vector dominated unlearning} (\Cref{prop:svdom}), and \textit{Phase 3: support-vector dominated generalization} (\Cref{thm:main_1}).

\textbf{Question 2.} When do these phases occur, and what do they imply for generalization?

\textbf{Answer.} The residual, i.e., the weights minus their growing max-margin component, converges at rate $O(t^{-\min(\lambda,1,\theta-1)})$ (\Cref{thm:rho_t_convergence}). Here $t$ is the iteration, $\lambda$ is the smallest eigenvalue of the weighted support-vector second-moment matrix, and $\theta$ is the smallest margin of the non-support vectors. This convergence guarantees that Phase 2 has begun (\Cref{lem:rhotconvuse}). We then lower-bound the generalization time in Phase 3 (\Cref{thm:main_3}), which shows the delay.

\textbf{Numerical experiments.} We validate these predictions in controlled experiments and in a deep network trained on MNIST (\Cref{sec:sec5}), where the three phases are clearly visible.

\section{Related Work}\label{sec:related}

\citet{power} first observed that models can fit the training set long before they generalize, with longer delays on smaller training sets. Proposed explanations include representation learning \citep{liu2022, gromov}, adaptive-optimizer instabilities \citep{thilak}, weight-norm dynamics \citep{liu2023}, lazy-to-rich transitions \citep{kumar}, regularization \citep{Junior}, shifts between early and late implicit biases \citep{lyu}, and changes in circuits, complexity or encodings \citep{nanda, DeMoss, manningcoe}. The phenomenon appears across architectures \citep{fan, humayun, doshi} and is linked to double descent \citep{nakkiran, davies} and robustness \citep{tan}. Tractable analyses cover linear estimators \citep{levi}, logistic regression \citep{beck}, rule-learning models \citep{zunkovic} and modular addition \citep{mohamadi}.

Our setting differs from these analyses. In \citet{levi}, the delay comes from the gap between the training loss and the test accuracy; ours persists in the accuracy itself and is set by the training data. In \citet{beck}, delayed generalization arises near the separability threshold; our data have a fixed margin, and the delay is controlled by the support-vector imbalance. \citet{lyu} need large initialization and weight decay to switch between early and late implicit biases. We use plain gradient descent from an arbitrary initialization, and the delay is caused solely by the mismatch between the learnable bias and the max-margin separator.

\section{Model Setup and Notation}\label{sec:model}
We consider a linearly separable training dataset $\mathcal{D} = \{(\x_1, y_1), (\x_2, y_2), \ldots, (\x_N, y_N)\}$ where features $\x_i\in\R^d$ and labels $y_i \in \{-1, 1\}$ and datapoint $(\x_i,y_i) \sim \mathcal{P}$. $N_+$ and $N_-$ respectively represent the sizes of the positive and the negative classes (which are represented by the sets $I^+$ and $I^-$) respectively. The data has a separating hyperplane that passes through the origin. Moreover, the dataset has a margin of separation $\gamma >0 $, meaning that all points lie at least $\gamma/2> 0$ on either side of this separating hyperplane. Formally,

\begin{assumption}[Separability and margin]\label{ass:sepmargin}
  The data is linearly separable about the origin and the maximum margin SVM solution passes through the origin. That is, $\exists \w_*\in \R^d$ such that $||\w_*||=1$ and $\forall \x_i, y_i\w_*^\top\x_i \geq \gamma/2$, and $\gamma/2$ is the largest such margin. Also, for any $\w_s\in \R^d$ and $b_s \in \R$ such that $\forall \x_i, y_i(\w_s^\top\x_i+b_s) \geq 1$ and $||\w_s||$ is minimum, $b_s = 0$.
\end{assumption}

Let $\S = \{\x_{s_1}, \x_{s_2}, \ldots, \x_{s_k}\}\subset\{\x_1,\ldots,\x_N\}$ be the support vectors, that is, the points with $y_{s_i}\hat{\w}^\top\x_{s_i}=1$, where $\hat{\w}$ is the maximum-margin vector in \Cref{eq:diffsvm}. By the KKT conditions there exist $\alpha_i > 0$ satisfying $\hat{\w}=\sum_{i=1}^k \alpha_i y_{s_i}\x_{s_i}$. The strict positivity follows from \Cref{eq:asymptotic_decomposition}. By \Cref{ass:sepmargin}, $||\hat{\w}||=2/\gamma$. The support vectors belonging to the positive and negative classes are denoted by $\S^+$ and $\S^-$ respectively.

Finally, we assume that the distribution has finite support.
\begin{assumption}[Finite Support]\label{ass:support}
 There exists a fixed $R > 0$ such that $||\x_i|| < R$ for all $i$.
\end{assumption}

\textbf{Model:} We consider a linear model with bias whose prediction for input $\x_i\in\R^d$ is given by $\hat{y}(\x_i;\w,b)=\w^\top \x_i+b$, where $\w\in\R^d$ and $b\in\R$ are the weight vector and bias respectively.

We define the training and test accuracies as follows. A point $(\x_i, y_i)$ is correctly classified if $y_i \hat{y}(\x_i; \w, b) > 0$. The training accuracy is defined as the empirical quantity $P(t) = \frac{1}{N} |\{i:y_i \hat{y}(\x_i; \w(t), b(t)) > 0 \}|$. The test accuracy is the population quantity $Q(t) = \mathbb{P}_{(\x, y) \sim \mathcal{P}_{test}}[y\hat{y}(\x; \w(t), b(t)) > 0]$, where the test distribution $\mathcal{P}_{test}$ is $\mathcal{P}$ or $\mathcal{P}_{conc}$ (\Cref{sec:sec5}).

\begin{definition}[Delayed generalization]\label{def:delgen}
    For a function $A : \mathbb{N}_0 \to [0, 1]$ with $\sup_u A(u) > 0$, given an $\epsilon \in [0, 1)$ define
    \begin{align*}
        T_A(\epsilon) := \inf \{t \in \mathbb{N}_0 : A(s) \geq (1-\epsilon)\sup_u A(u) \forall s \geq t \}.
    \end{align*}
    Then, we define the training time by $T_P(\epsilon)$ and the generalization time to be $T_Q(\epsilon)$.
\end{definition}

\textbf{Loss and Update:} We are interested in minimizing the following loss function
\begin{align*}
    \mathcal{L}(\w,b) = \sum_{i=1}^N l(y_i(\w^\top\x_i+b))
\end{align*}
where $l(x) = \exp({-x})$ is the exponential loss. Note that a variety of losses have the same tail behavior as the exponential loss which makes this a natural object to study. We allow for arbitrary initialization of the weights and the bias $(\w(0), b(0))$. We use gradient descent to optimize this loss with constant step size $\eta\le\frac{1}{2N(R^2+1) \mathcal{L}(\w(0), b(0))}$. This is a sufficient condition which makes the exponential loss decrease monotonically when $\|\x_i\|\le R$, \Cref{ass:support}).

\begin{figure*}[t]
    \centering

    \begin{subfigure}[t]{0.48\textwidth}
        \centering
        \includegraphics[width=\linewidth]{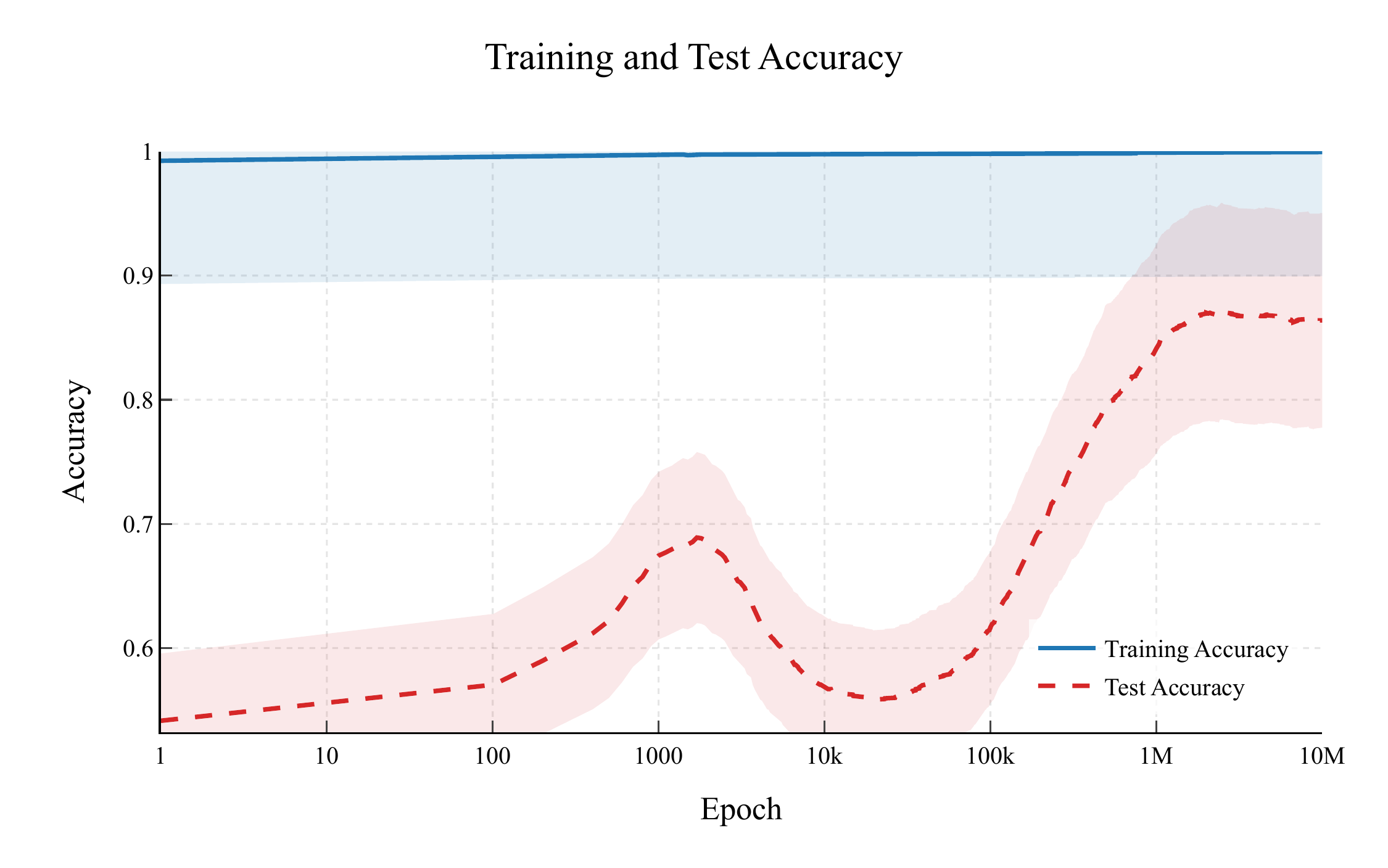}
        \caption{Distribution shift}
        \label{fig:accuracy_dist_shift_dim_2}
    \end{subfigure}
    \begin{subfigure}[t]{0.48\textwidth}
        \centering
        \includegraphics[width=\linewidth]{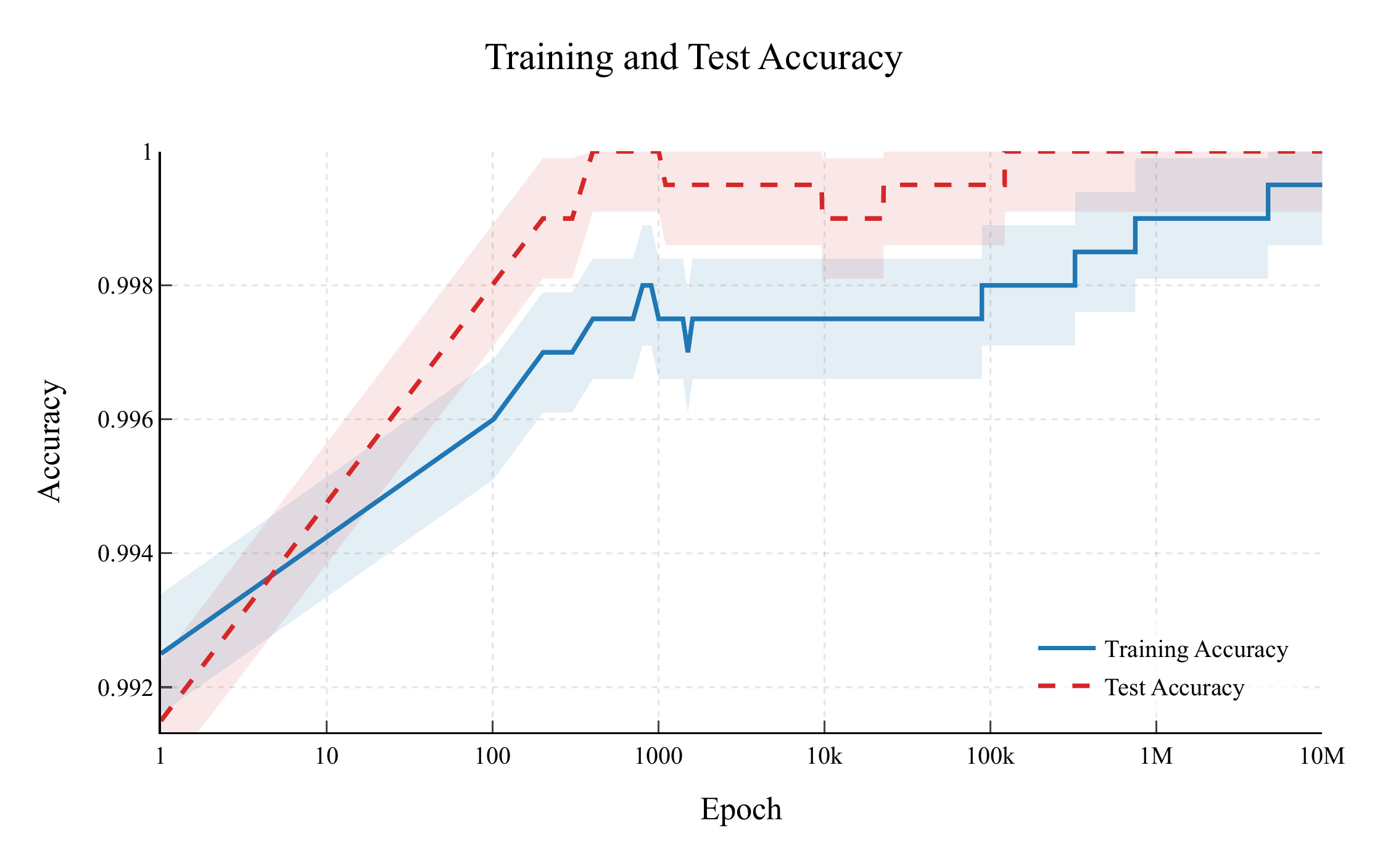}
        \caption{Same distribution as training}
        \label{fig:accuracy_same_as_train_dim_2}
    \end{subfigure}

    \vspace{2\baselineskip}
    \caption{Performance on test data under (a) Distribution Shift (left) and (b) same distribution as train data (right).}
    \label{fig:delayed convergence}
\end{figure*}

\section{Model Learns in Three Phases}\label{sec4}
In this section we show that the learning can be categorized into three phases. In the first phase, all the data affect the dynamics and hence we call this the population dominated phase. In the second and third, the support vectors dominate the dynamics.

In the setting we consider, \citet{soudry} showed that
\begin{align}\label{eq:asymptotic_decomposition}
    \w(t)&=\hat{\w}\log t+\brho(t),\\
    \eta \exp(-\x_n^\top\tilde{\w})&=\alpha_n\ \ \forall n\in\S. \nonumber
\end{align}
Here $\hat{\w}$ is the maximum-margin vector, and in the second equation the labels are absorbed into the data ($\x_n$ stands for $y_n\x_n$). With a learnable bias, these results apply to the augmented data $y_n(\x_n,1)$, and $\tilde{\w}$ denotes the weight part of the limit. The direction of $\w(t)$ converges to the maximum-margin direction. The residual $\brho(t)$ is bounded for almost all separable datasets. When the support vectors $\S$ span the data, $\brho(t)$ (which is always bounded by a constant $B$) converges to the vector $\tilde{\w}$ characterized by the second equation.

\subsection{Theoretical Result On Learning Phases} \label{ourtheory}
We first show that the augmented SVM that \cite{soudry} analyze is the same as the true SVM under our assumptions.
\begin{theorem}\label{thm:svmsolnequivalence}
    Under \Cref{ass:sepmargin} the SVM solution given by
    \begin{align} \label{eq:svm}
    \w_{SVM} = \arg \min ||\w||^2  \text{   s.t.   }y_i(\w^\top\x+b)\geq 1 
\end{align}
is equal to
\begin{align}\label{eq:diffsvm}
    \hat{\w} = \arg \min ||\w||^2 + b^2  \text{   s.t.   } y_i(\w^\top\x+b)\geq 1.
\end{align}
\end{theorem} 
The proof can be found in \Cref{app:thm:svm}.

Though the three phases have gradual transitions, we define them below with appropriate thresholds for clarity.

\begin{definition}\label{def:phases}
Define the population weights $A^+(t)=\sum_{i\in I^+}\exp(-\w(t)^\top\x_i)$ and $A^-(t)= \sum_{i\in I^-}\exp(\w(t)^\top\x_i)$, and their restrictions to the support vectors, $A_S^+(t)=\sum_{i\in \S^+}\exp(-\w(t)^\top\x_i)$ and $A_S^-(t)=\sum\limits_{i\in \S^-}\exp(\w(t)^\top\x_i)$. Fix a threshold $\vartheta\in(0,1)$ and let
\begin{align*}
T_1(\vartheta)&=\inf\Big\{T: \min\Big(\frac{A_S^+(t)}{A^+(t)},\frac{A_S^-(t)}{A^-(t)}\Big)\ge1-\vartheta\ \forall t\ge T\Big\},\\
T_2(\vartheta)&=\min\Big\{t\ge T_1(\vartheta):\ \frac{|b(t)|}{||\w(t)||}=\sup_{s\ge T_1(\vartheta)}\frac{|b(s)|}{||\w(s)||}\Big\},
\end{align*}
whenever these quantities are well defined. Then Phase 1 is $[0,T_1)$: population dominated, Phase 2 is $[T_1,T_2)$: support-vector-dominated unlearning, and Phase 3 is $[T_2,\infty)$: support-vector-dominated generalization.
\end{definition}

We show that $T_1$ is always well defined, and that $T_2$ is well defined whenever $b(t)\neq0$ for some $t\ge T_1$ (refer to \Cref{prop:svdom} and the discussion after \Cref{thm:main_1}). First we make sense of the definition. Observe that $0 \leq A_S^+(t) \leq A^+(t)$ and $0 \leq A_S^-(t) \leq A^-(t)$ always because the set of all support vectors is a subset of all the points in the class and all the terms are positive. Also as the dynamics proceed, the contribution of the support vectors becomes more important and we have $A_S^+(t)/A^+(t) \to 1$ and $A_S^-(t)/A^-(t) \to 1$ (\Cref{prop:svdom}(a)). The ratio becoming close to unity is therefore an indication that the support vectors dominate the dynamics. Beyond the population phase, unlearning refers to the bias moving away from its generalizing solution. The definition of $T_2$ captures this naturally.

Phase~1 is dictated by the entire training dataset, while the weights follow the decomposition of \citet{soudry} recalled in \Cref{sec:intro}. 
It turns out that the convergence of the residual term \(\brho(t)\) to its limiting value \(\tilde{\w}\) is sufficient for Phase~2 to have begun. This is shown in \Cref{lem:rhotconvuse}.

It is therefore consequential to understand the rate of this convergence. This rate bounds how long the population phase can last. We bound it in the theorem below.



\begin{theorem}\label{thm:rho_t_convergence}
Consider $\H = \sum\limits_{\x_n \in \S} \alpha_n \x_n\x_n^T$, with labels absorbed into $\x_n$. Suppose $\lambda > 0$ is the minimum eigenvalue of $\H$ on the span of the support vectors. Suppose $\theta > 1$ is the minimum margin $y_n\hat{\w}^\top\x_n$ of the non-support vectors. Assume $\lambda \notin \{1, \theta -1\}$. Define $\r(t) = \w(t) -\hat{\w}\log t - \tilde{\w} = \brho(t) - \tilde{\w}$. Then there exist constants $C$ and $T_0$ such that $||\r(t)||\le Ct^{-p}$ for all $t > T_0$, where $p=\min(\lambda, 1, \theta-1)$.
\end{theorem}

The proof can be found in \Cref{app:thm:rho_t_convergence}.

Notice that while $\w(t)$ converges to its true direction as $O(1/\log t)$ (\cite{soudry}), the convergence of $\brho(t)$ is much faster. We show below that Phase $1$ has ended by the time $\brho(t)$ has converged.

\begin{lemma}\label{lem:rhotconvuse}
Let $\epsilon>0$ and let
\begin{align*}
t_0(\epsilon):=\min\big\{t\ge1:\ ||\brho(s)-\tilde{\w}||\le\epsilon\ \ \forall s\ge t\big\}
\end{align*}
be the first time after which the residual stays $\epsilon$-close to its limit. Define $\vartheta(\epsilon):=N\,\exp(2R(B+\epsilon))\,t_0(\epsilon)^{1-\theta}$. If $\vartheta(\epsilon)<1$, then $T_1\big(\vartheta(\epsilon)\big)\ \le\ t_0(\epsilon)$. That is, by the time the residual has converged to within $\epsilon$, the support vectors carry at least a fraction $1-\vartheta(\epsilon)$ of both $A^+(t)$ and $A^-(t)$, and Phase~2 has begun.
\end{lemma}
The proof can be found in \Cref{app:cor:T1suff}.

Hence the convergence of $\brho(t)$ is a sufficient condition for the onset of Phase 2. We are now ready to discuss the bias dynamics in the three phases. We first show that, at every step, the bias moves toward a population quantity $b^*(t) = \frac{1}{2}\log \frac{A^+(t)}{A^-(t)}$. 

\begin{lemma}\label{lem:popdyn}
    Suppose $b^*(t) = \frac{1}{2}\log \frac{A^+(t)}{A^-(t)}$. For every $t \in \mathbb{N}_0$, gradient descent on the bias gives
    \begin{align}\label{eq:biasdiffeq}
        b(t+1) - b(t) &= \eta(A^+(t)\exp(-b(t)) - A^-(t)\exp(b(t)))\\
        &= -2\eta\sqrt{A^+(t)A^-(t)} \sinh(b(t) - b^*(t)). \nonumber
    \end{align}
\end{lemma}
The proof can be found in \Cref{app:lem:popdyn}.

This naturally extends to the late time dynamics where we see an analogous support vector quantity $b_S^*(t) = \frac{1}{2}\log \frac{A_S^+(t)}{A_S^-(t)}$. In the support vector dominated phases, we show that $b^*(t)$ is close to $b_S^*(t)$. This is formalized in the below proposition.

\begin{proposition}\label{prop:svdom}
Let \Cref{ass:sepmargin} and \Cref{ass:support} hold, and let $C,p$ be as in \Cref{thm:rho_t_convergence}. Define $b^\star(t):=\frac12\log\frac{A^+(t)}{A^-(t)}$ and $b_S^\star(t):=\frac12\log\frac{A_S^+(t)}{A_S^-(t)}$. Then for all $t\ge1$,
\begin{enumerate}
\item[(a)] $\max\left(1-\frac{A_S^+(t)}{A^+(t)},  1-\frac{A_S^-(t)}{A^-(t)}\right) \leq N \exp(2RB)t^{1-\theta}$.
\item[(b)] If $\min\left(\frac{A_S^+(t)}{A^+(t)},  \frac{A_S^-(t)}{A^-(t)}\right)\ge1-\vartheta$, then $ |b^\star(t)-b_S^\star(t)|\le\frac12\log\frac1{1-\vartheta}$.
\end{enumerate}
\end{proposition}
The proof can be found in \Cref{app:prop:svdom}.

Part $(a)$ above shows how quickly Phase $2$ begins. Now, \Cref{thm:rho_t_convergence} can be intuitively understood. Since the residual converges no faster than $O(t^{-(\theta-1)})$, which is the rate at which the bound in (a) vanishes, the support vectors dominate by the time the residual has converged. This also shows the existence and finiteness of $T_1(\vartheta)$. Part $(b)$ shows that the support vectors dominate the steady state bias.

In the late time regime, $tA_S^+(t)$ and $tA_S^-(t)$ converge to the fixed quantities $\tilde{A}_S^+$ and $\tilde{A}_S^-$ defined as follows. Let $\tilde{A}_S^+=\sum_{i\in\S^+}\exp(-\tilde{\w}^\top\x_i)$, $\tilde{A}_S^-=\sum_{i\in\S^-}\exp(\tilde{\w}^\top\x_i)$ and $\delta=\sqrt{\tilde{A}_S^-/\tilde{A}_S^+}$. The dynamics in this regime are also dependent on the minimum margins of the non-support vectors. To this end, define
\begin{align*}
    \theta^+&=\min_{i\in I^+\setminus\S^+}\hat{\w}^\top\x_i>1\\
    \theta^-&=\max_{i\in I^-\setminus\S^-}\hat{\w}^\top\x_i<-1.
\end{align*}
Recall that by \Cref{thm:rho_t_convergence}, there are constants $C>0$ and $T_0$ such that $\|\brho(t)-\tilde{\w}\|\le kt^{-p}$ for $t\ge T_0$, where $p=\min(\lambda,1,\theta-1)$. Let $C_2=|I^+\setminus\S^+|\exp(RB)$ and $C_3=|I^-\setminus\S^-|\exp(RB)$, and suppose that $|b(t)|\le K$ for all $t\ge1$. We know that the bias is always bounded because we can apply \Cref{eq:asymptotic_decomposition} to the bias term. 

For $t_0\ge T_0$, let $\bar A_{\S}^+, \bar A_{\S}^-,\underline{A}_{\S}^+, \underline{A}_{\S}^-$ be the coefficients defined in \Cref{app:thm:main_1}, which can also be found in \Cref{app:notation}.
Let $\bar\delta=\sqrt{\bar A_{\S}^-/\bar A_{\S}^+}$, $\underline\delta=\sqrt{\underline{A}_{\S}^-/\underline{A}_{\S}^+}$ and $M=\max(K,|\log\bar\delta|,|\log\underline\delta|)$.

\begin{theorem}\label{thm:main_1}
Let \Cref{ass:sepmargin,ass:support} hold, and let $t_0\ge T_0$ satisfy $t_0\ge\eta(\bar A_{\S}^++\bar A_{\S}^-+\underline{A}_{\S}^++\underline{A}_{\S}^-)\exp(2M)$. Let $\bar b(t)$ and $\underline b(t)$ be defined by $\bar b(t_0)=\underline b(t_0)=b(t_0)$ and
\begin{align}\label{eq:biasdyn}
\bar b(t+1)&=\bar b(t)+\frac{\eta}{t}\left(\bar A_{\S}^+\exp(-\bar b(t))-\bar A_{\S}^-\exp(\bar b(t))\right),\\
\underline b(t+1)&=\underline b(t)+\frac{\eta}{t}\left(\underline{A}_{\S}^+\exp(-\underline b(t))-\underline{A}_{\S}^-\exp(\underline b(t))\right).\nonumber
\end{align}
Then for all $t\ge t_0$,
\begin{align}\label{eq:biasevol}
\underline b(t)\le b(t)\le\bar b(t),
\end{align}
where $\bar b(t)$ converges monotonically to $-\log\bar\delta$ with $|\bar b(t)+\log\bar\delta|\le|b(t_0)+\log\bar\delta|(t_0/t)^{2\eta\sqrt{\bar A_{\S}^+\bar A_{\S}^-}}$, and likewise $\underline b(t)$ converges to $-\log\underline\delta$. Consequently, $b(t)$ converges to $b_\infty=-\log\delta$.
\end{theorem}
The proof can be found in \Cref{app:thm:main_1}.

Since $|b(t)|\le K$ and $||\w(t)||\to\infty$, the ratio $|b(t)|/||\w(t)||$ tends to $0$. Hence, whenever $b(t)\neq0$ for some $t\ge T_1(\vartheta)$, the supremum in \Cref{def:phases} is positive and attained, and $T_2(\vartheta)$ is well defined. 

\Cref{eq:biasdiffeq} carries rich meaning. For a dataset with balanced classes, $A^+(0)=N_+=N_-=A^-(0)$, and heuristically $A^+ \approx A^-$ early in training. This is because the training population is symmetric on an average and adding the cost $\exp({\pm \w(t)^\top \x_i})$ over each class for every data point will yield similar values. Following a similar argument, for a dataset with class imbalance towards $+1$ ($-1$), $A^+ > A^-$ ($A^+ < A^-$). 

Equipped with the above observation we now discuss the Phases 1, 2 and 3.

\subsection{Phase 1: Population dominated}\label{4.3}

In the early dynamics of \Cref{eq:biasdiffeq}, $A^+$ and $A^-$ change with time as $\w$ evolves. We also observe from \Cref{eq:asymptotic_decomposition} that since $||\w||$ evolves approximately as $\mathcal{O}(\log t)$, and both $A^+$ and $A^-$ depend on (largely positive) inner products with $\w$, on an average both decrease with time. Now for different class sizes in the dataset, we heuristically expect the following evolutions of the early time dynamics.

\begin{itemize}
    \item[1.] \textbf{Classes balanced}: In this case, $A^+ \approx A^-$. Hence, on an average $b(t+1)-b(t) \propto \exp(-b) - \exp(b)$. This quickly brings $b$ to $0$ regardless of where it is initalized. Hence throughout the early time dynamics, $b$ remains near zero.
    \item[2.] \textbf{Class imbalance towards $+1$ ($-1$)}: In this case $A^+ > A^-$ ($A^+ < A^-$). Now the early time dynamics no longer force $b$ to $0$. Instead, the steady state solution is $b = 0.5\log (A^+/A^-) > 0$ ($< 0$). Hence regardless of where it is initialized, the bias is pulled to this value.
\end{itemize}

\begin{figure*}[!hbt]
    \centering

    \begin{subfigure}[t]{0.29\textwidth}
        \centering
        \includegraphics[width=\linewidth]{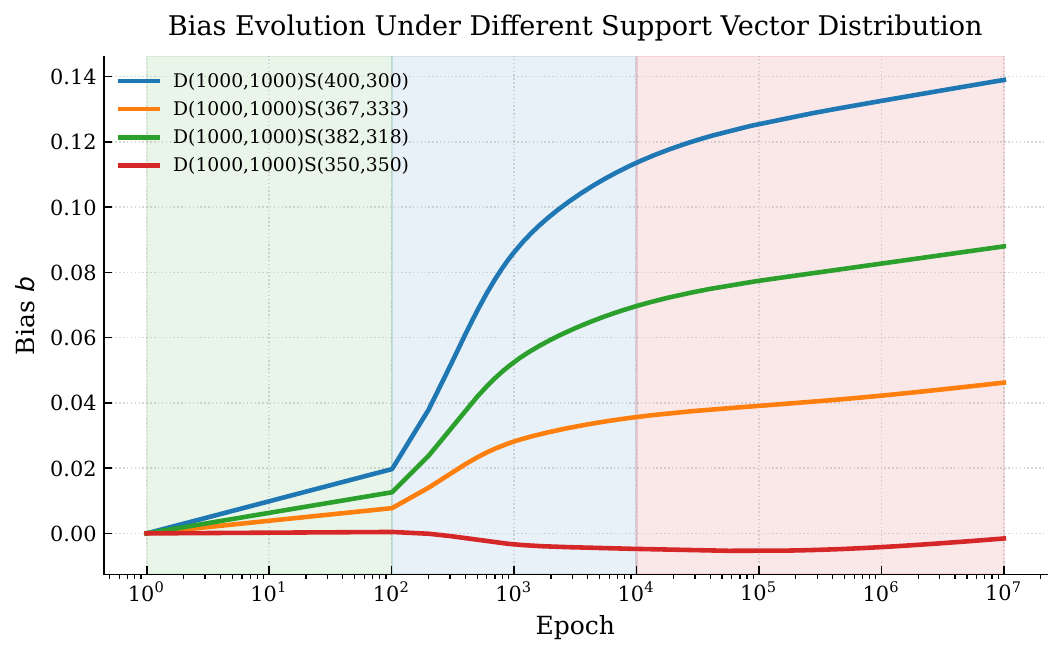}
        \caption{Bias dynamics}
        \label{fig:all_symm_bias}
    \end{subfigure}
    \begin{subfigure}[t]{0.29\textwidth}
        \centering
        \includegraphics[width=\linewidth]{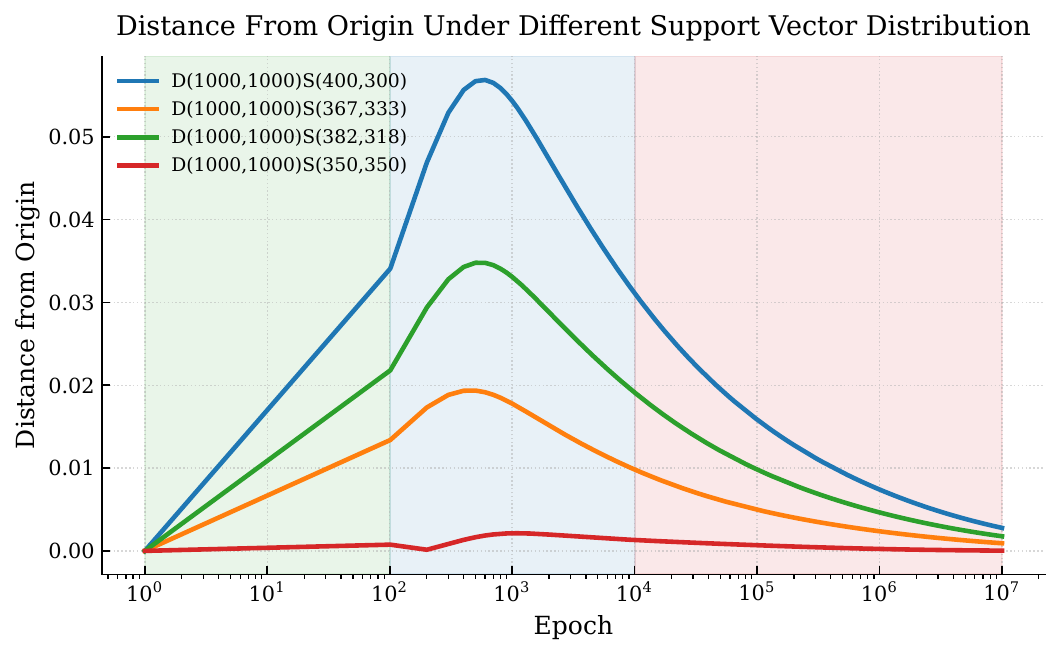}
        \caption{Distance from origin}
        \label{fig:all_symm_dfo}
    \end{subfigure}
    \begin{subfigure}[t]{0.32\textwidth}
        \centering
        \includegraphics[width=\linewidth]{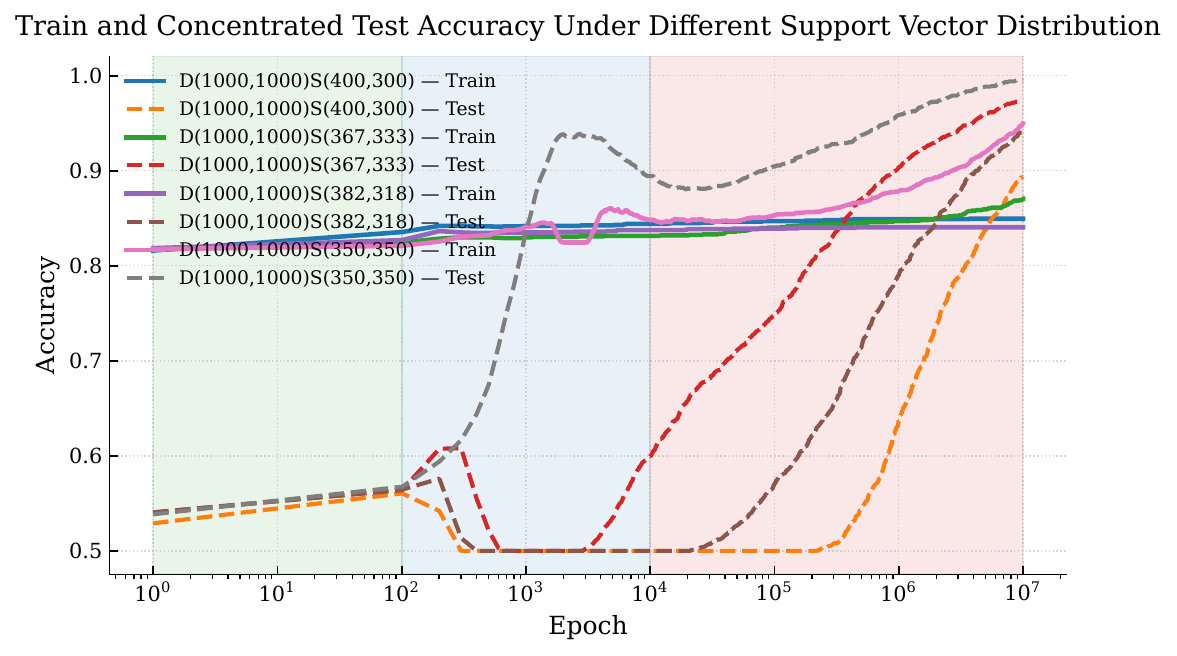}
        \caption{Train and conc. Test accuracy}
        \label{fig:all_symm_tta}
    \end{subfigure}
    \vspace{2\baselineskip}
   \caption{\textbf{Support Vector distributions affect convergence time.} (a) Bias (left) (b) Distance from origin (middle) (c) train and concentrated test accuracy (right) versus number of epochs under different numbers of positive and negative support vectors, where D(1000,1000) mean 1000 positive class points and 1000 negative class points. Similarly S(350,350) means 350 support vectors for positive class and 350 support vectors for negative class. The three shaded regions refer to phase 1, phase 2, phase 3 respectively in all the plots. }
    \label{fig:4}
    \label{fig:all_symm_comparison}
\end{figure*}
\subsection{Phases 2 and 3: Support Vector dominated}\label{4.4}

In the late time dynamics of \Cref{eq:biasdiffeq}, the non-support vectors contribute a vanishing fraction of $A^+$ and $A^-$ (\Cref{prop:svdom}(a)), so $A_S^+/A^+\to1$ and $A_S^-/A^-\to1$. Similar to the observation about $A^+$ and $A^-$, we make the observation that $A_S^+$ and $A_S^-$ decay with time. This decay is $\mathcal{O}(1/t)$ as explicitly seen in this learning phase from applying \Cref{eq:asymptotic_decomposition} to \Cref{eq:biasdiffeq} (see \Cref{eq:biasdyn}).

Now, regardless of class balance, the support vectors in the training data could be symmetrically or asymmetrically (significantly more support vectors belonging to one class than the other) balanced. Symmetric balance does not necessarily imply distributional symmetry but for large enough number of support vectors this is approximately true. In each of these cases the late time dynamics evolves differently.

\begin{itemize}
    \item[1.] \textbf{Symmetrically balanced support vectors}: In this case, $\tilde{A}_S^+ \approx \tilde{A}_S^-$. Similar to the first case in early time dynamics, $b(t+1)-b(t) \propto \exp(-b) - \exp(b)$. Regardless of the population class balance, this regime drives $b$ to $b_\infty=-\log\delta\approx0$ (\Cref{thm:main_1}), slowly, since the effective step size decays as $1/t$.
    \item[2.] \textbf{Asymmetrically more $+1$ ($-1$) class support vectors}: In this case, $\tilde{A}_S^+ > \tilde{A}_S^-$ ($\tilde{A}_S^+ < \tilde{A}_S^-$). The late time dynamics therefore drive $b$ to $b_\infty = 0.5\log (\tilde{A}_S^+/\tilde{A}_S^-) > 0$ ($< 0$) regardless of the early time dynamics (\Cref{thm:main_1}).
\end{itemize}
The above discussion intuitively explains why the late time limit of \Cref{eq:biasevol} is $b_\infty = -\log \delta$. Now we analyze the effect of these learning phases on the test accuracy.

\subsection{Effect of learning phases on generalization}\label{4.5}

\begin{itemize}
    \item \textbf{Population dominated phase}: In the early time dynamics $\w$ converges to $\hat{\w}$ in direction which causes some increase in test accuracy. The bias either remains at $0$ for class balanced datasets or goes to $b = 0.5\log (A^+/A^-)$ for class imbalanced datasets. This slightly worsens the accuracy but the effect of both these dynamics in tandem keeps the test accuracy roughly the same.
    \item \textbf{Support Vector dominated unlearning regime}: In contrast to the previous phase, in the late time dynamics the direction of $\w$ changes only slowly, as $O(1/\log t)$. The asymmetry between $A_S^+ \exp(-b)$ and $A_S^- \exp(b)$ drives $b$ away from the generalizing solution $b=0$. The test performance therefore definitively worsens in this regime. We call this the unlearning regime.
    \item \textbf{Support Vector dominated generalization regime}: While $b$ converges to $0.5\log (\tilde{A}_S^+/\tilde{A}_S^-)$, notice that $||\w||$ increases logarithmically. Since the hyperplane is nearly parallel to the maximum margin separator, test performance is governed mainly by the distance of the hyperplane from the origin (i.e. $|b(t)|/||\w(t)||$), which goes to zero as $\mathcal{O}(1/\log t)$. The remaining direction error enters \Cref{thm:main_3} through the term $RB$. We call this the generalization regime.
\end{itemize}

\subsection{Theoretical Result on Generalization Time}\label{tgroksec}

\begin{theorem}\label{thm:main_3}
Let \Cref{ass:sepmargin} and \Cref{ass:support} hold, and suppose $|b(t)|\le K$ for all $t\ge1$. Let the test distribution be $\mathcal{P}_{conc}$: the classes are balanced, test points satisfy $||\x||\le R$, and within each class the margin $s=y\hat{\w}^\top\x/||\hat{\w}||$ is uniform on $[\gamma/2,\alpha\gamma]$. Fix $\epsilon\in(0,\frac12)$ and $\vartheta\in(0,1)$, and suppose that $|b(t)|\ge(1-\vartheta)|\log\delta|$ for all $t\ge T_2=T_2(\vartheta)$. Then
\begin{align*}
\log T_2&<\frac{(1-\vartheta)|\log\delta|-RB}{1+2\epsilon(2\alpha-1)}\\
\implies T_Q(\epsilon) &\ge \exp\left(\frac{(1-\vartheta)|\log\delta|-RB}{1+2\epsilon(2\alpha-1)}\right).
\end{align*}
\end{theorem}
The proof can be found in \Cref{app:thm:main_3}.

The above theorem shows that if the support vector dominated phases start early enough then the generalization is delayed. This is because in Phase 3, the bias is near $-\log \delta$. So the distance of the hyperplane from the origin is approximately $\gamma |\log \delta|/(2\log t)$. This is related to the population misclassification. This results in a generalization time exponential in the steady state bias.

To demonstrate the validity of the above theorem we use the dataset with balanced classes and different values of $\log \delta$. We control this by planting a different ratio of support vectors in each class which is a measure of $\delta$. The plots are shown in \Cref{fig:4}. The implementation details can be found in \Cref{app:experimental_setup}. 

As theoretically predicted, we observe that a higher imbalance in the support vectors (which corresponds to a larger $|\log\delta|$) causes the bias to stagnate at a higher value. This causes the hyperplane to move further from the origin before learning the generalizing solution, increasing convergence time.

\section{Dataset and Experimental Results}\label{sec:sec5}
\begin{figure*}[!hbt]
    \centering

    \begin{subfigure}[b]{0.29\textwidth}
        \centering
        \includegraphics[width=\linewidth]{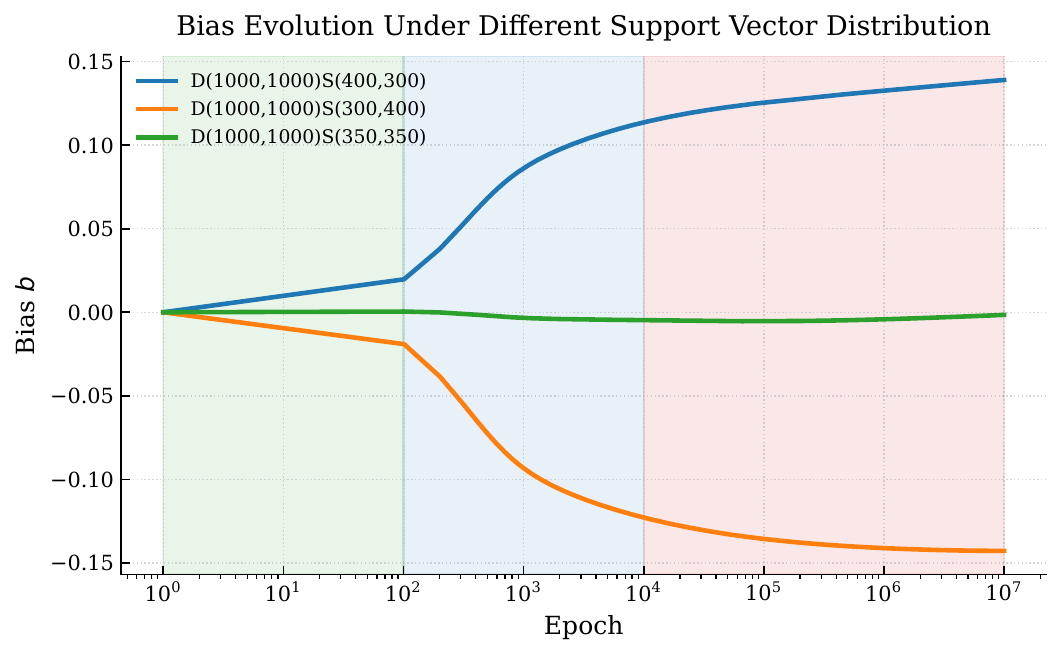}
    \end{subfigure}
    \begin{subfigure}[b]{0.29\textwidth}
        \centering
        \includegraphics[width=\linewidth]{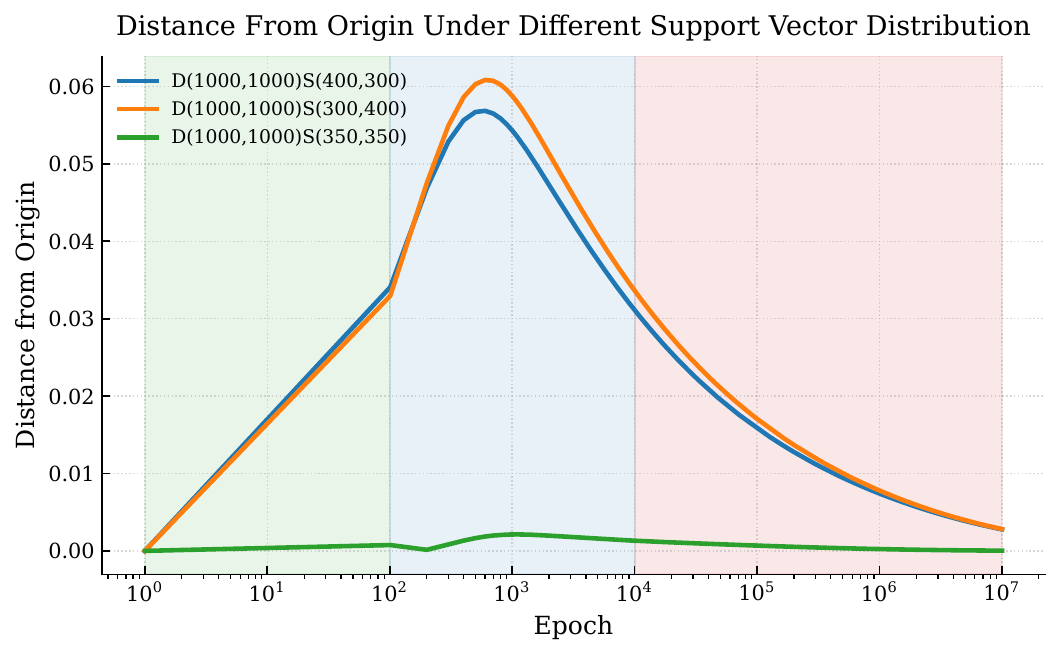}
    \end{subfigure}
    \begin{subfigure}[b]{0.32\textwidth}
        \centering
        \includegraphics[width=\linewidth]{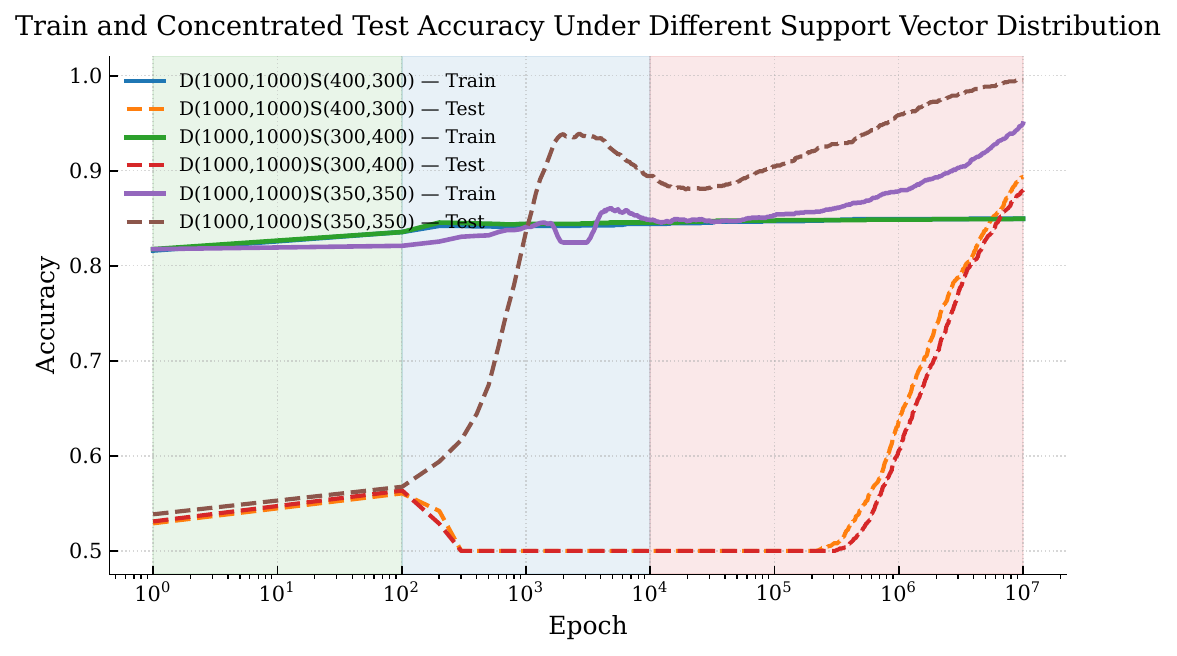}
    \end{subfigure}


    \begin{subfigure}[b]{0.29\textwidth}
        \centering
        \includegraphics[width=\linewidth]{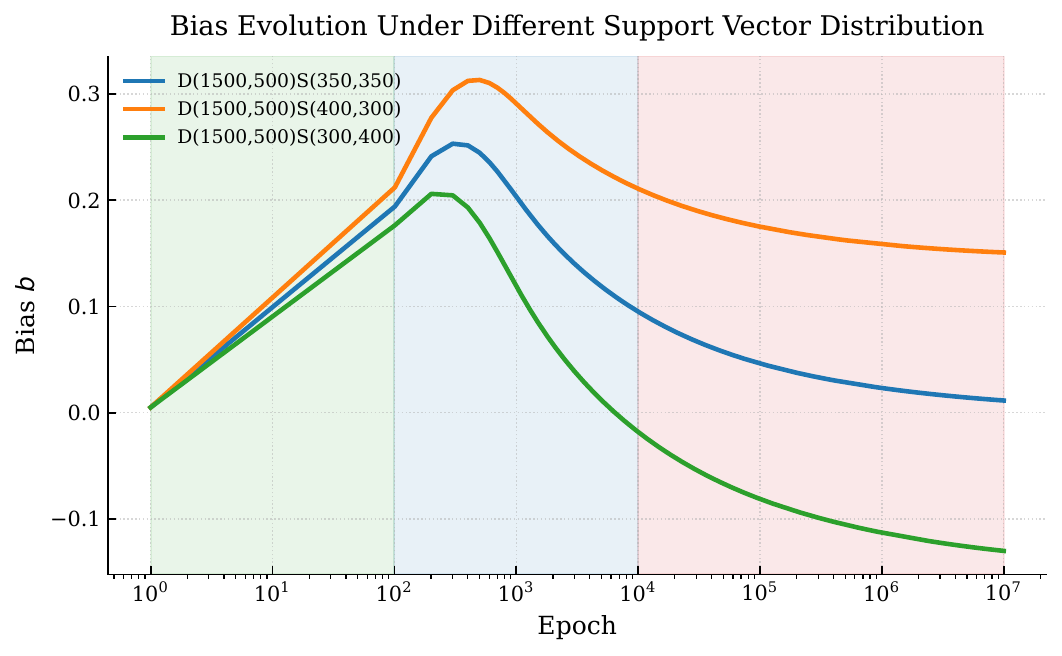}
    \end{subfigure}
    \begin{subfigure}[b]{0.29\textwidth}
        \centering
        \includegraphics[width=\linewidth]{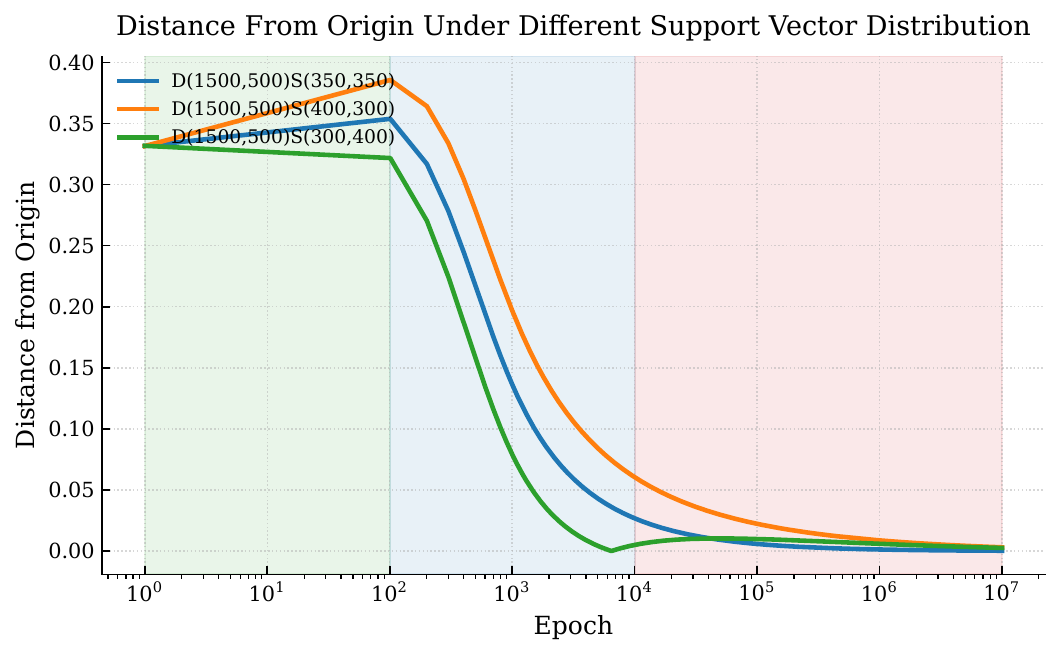}
    \end{subfigure}
    \begin{subfigure}[b]{0.32\textwidth}
        \centering
        \includegraphics[width=\linewidth]{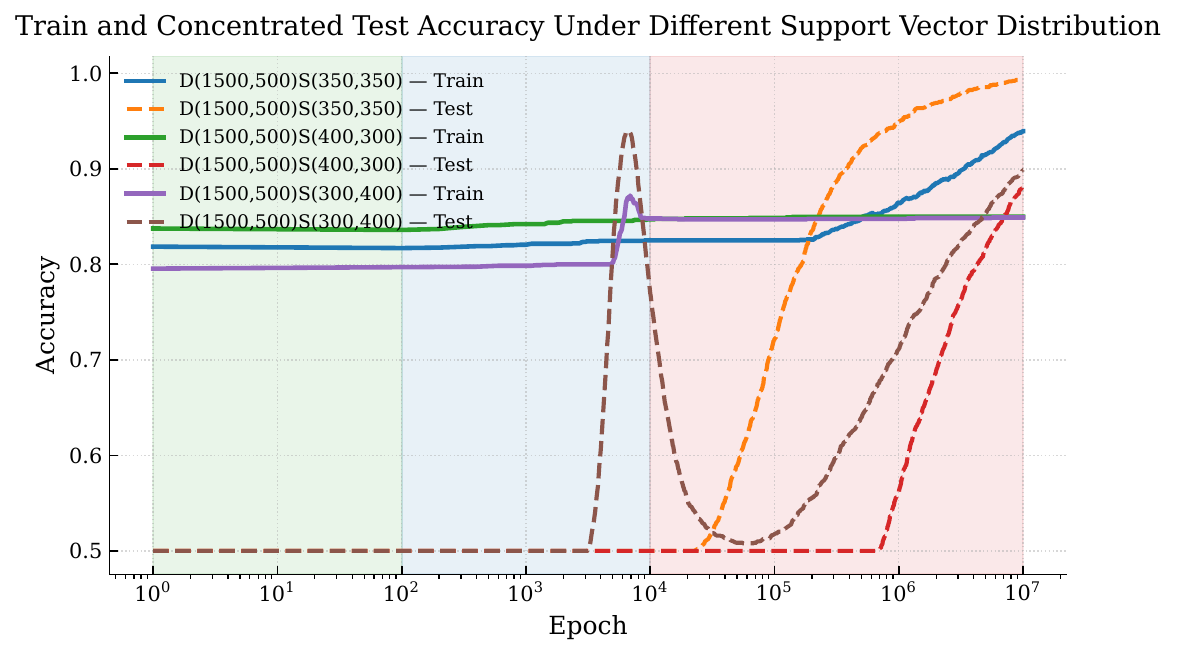}
    \end{subfigure}


    \begin{subfigure}[b]{0.29\textwidth}
        \centering
        \includegraphics[width=\linewidth]{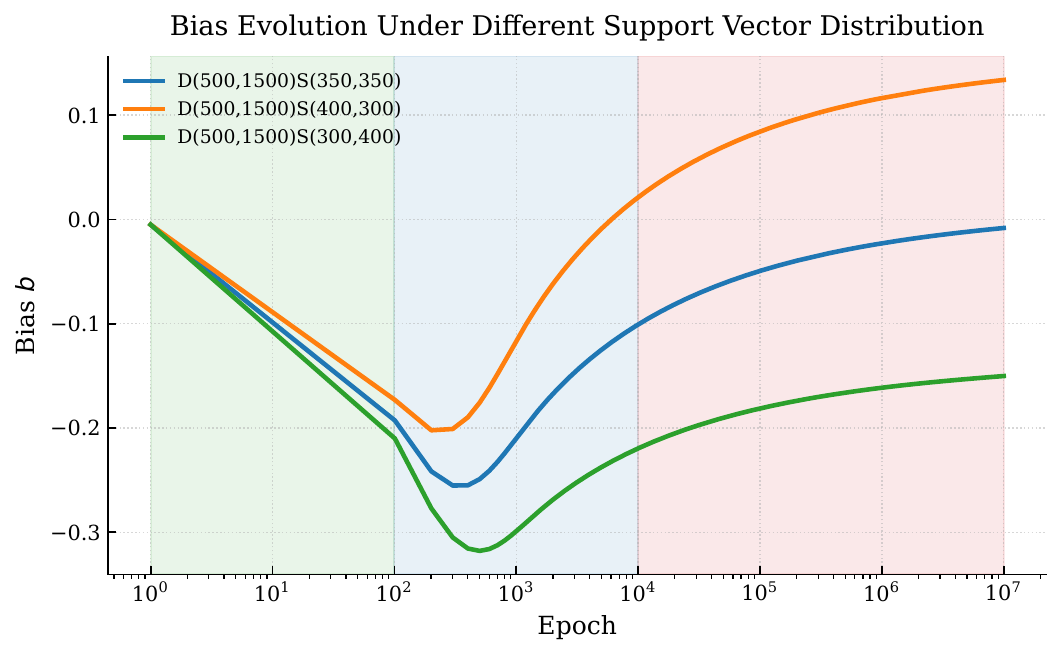}
    \end{subfigure}
    \begin{subfigure}[b]{0.29\textwidth}
        \centering
        \includegraphics[width=\linewidth]{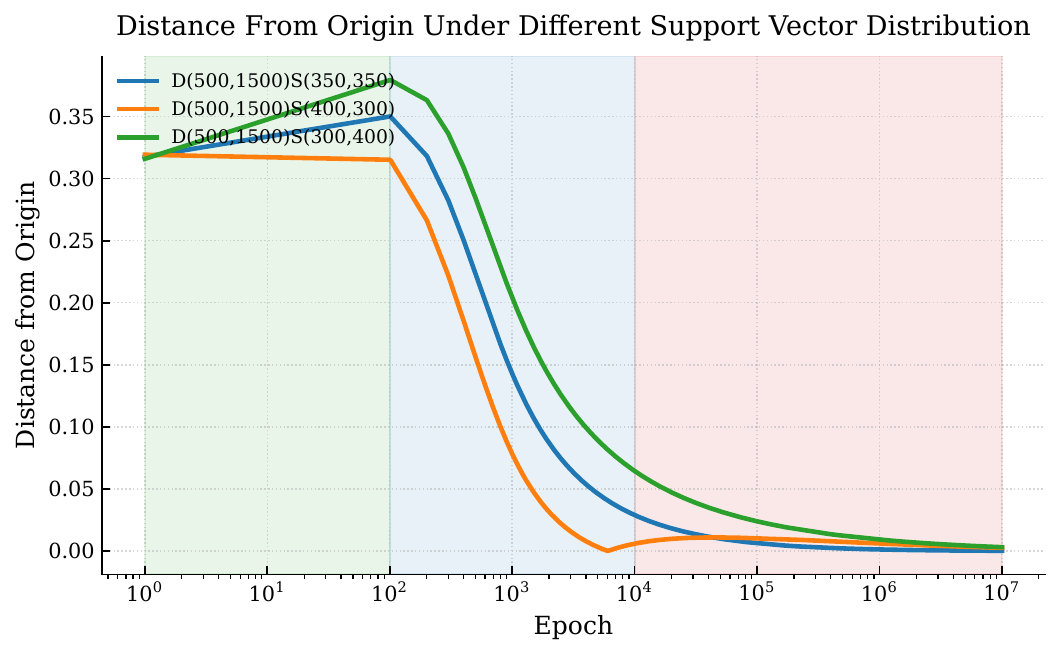}
    \end{subfigure}
    \begin{subfigure}[b]{0.32\textwidth}
        \centering
        \includegraphics[width=\linewidth]{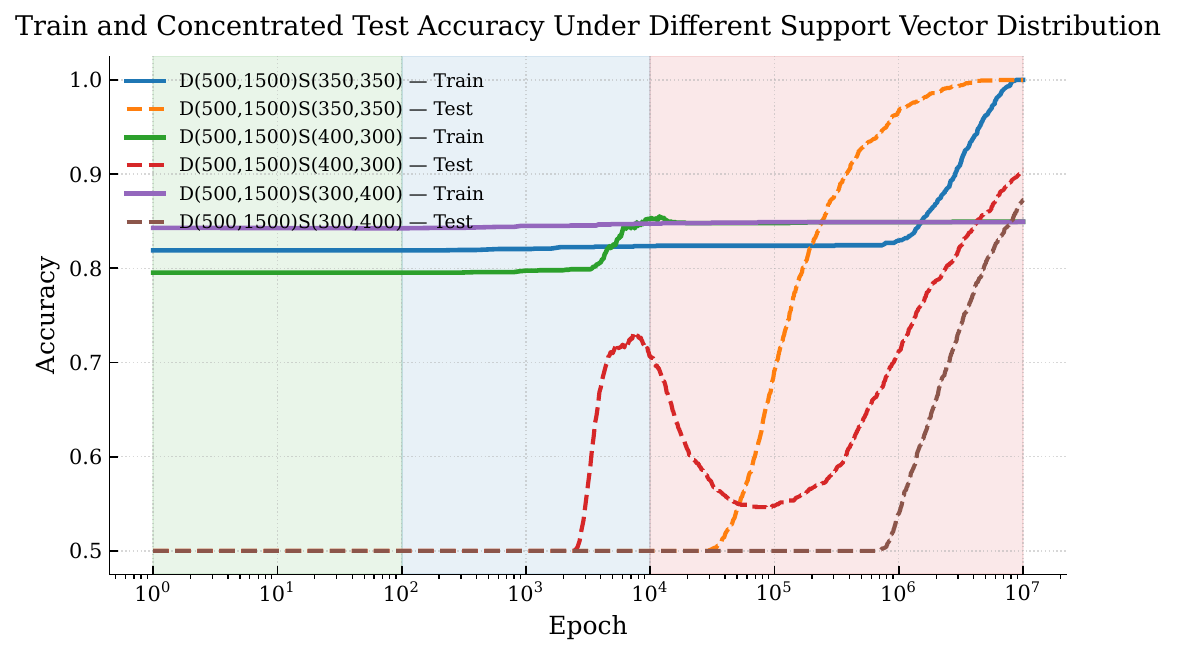}
    \end{subfigure}
    \vspace{2\baselineskip}
    \caption{Evolution of (a) the bias $b(t)$ (left), (b) the distance of the separating hyperplane from the origin $|b(t)|/\|\w(t)\|$ (middle), and (c) training and concentrated test accuracies (right), as functions of training epochs, for varying degrees of positive and negative class imbalance. Notation is the same as \Cref{fig:4}.}
    \label{fig:symm_asymm+_asymm-}
\end{figure*}

\subsection{Dataset distribution details} 

\begin{itemize}
    \item \textbf{Test distribution $\mathcal{P}$}[Standard Test]: 
    Samples $\x \sim \mathcal{P}$ are drawn uniformly from the $d$-dimensional hypercube $[-5,5]^d$. 
    The separating hyperplane is defined by $\w_{\text{true}}^\top \x = 0$
    where, $\w_{\text{true}} = \frac{1}{\sqrt{d}}[1,\ldots,1]^\top \in \mathbb{R}^d$.
    The dataset satisfies a margin condition with margin $\gamma = 10^{-3}$. 
    Labels are assigned as
    \begin{align*}
        y =
        \begin{cases}
            +1, & \text{if } \w_{\text{true}}^\top \x \ge \gamma/2, \\
            -1, & \text{if } \w_{\text{true}}^\top \x \le -\gamma/2.
        \end{cases}
    \end{align*}

    \item \textbf{Sensitive distribution $\mathcal{P}_{conc}$} [Concentrated Test]:
    The distribution $\mathcal{P}_{conc}$ concentrates mass near the decision boundary. 
    Samples satisfy $\gamma/2 \le \w_{\text{true}}^\top \x \le \alpha\gamma$, for $y = +1$, and  $-\alpha\gamma \le \w_{\text{true}}^\top \x \le -\gamma/2$, for $y = -1$, where $\alpha$ controls the severity of the distribution shift and is set to $\alpha = 10$. We refer to $\alpha$ as the \emph{sensitivity parameter}.
\end{itemize}
We conduct studies on these distributions as the sources for our train and concentrated test sets for $d=2$ (small dimension, easy to visualize hyperplane movement). ~\Cref{fig:delayed convergence} presents the simulation results for the above setup.\\
To connect  the theory to deep learning, we further extend our analysis towards training a deep network with a single output neuron for binary classification , on visually similar digits $4$ and $9$. 

\subsection{Analysis of results}\label{sec:analysis}
\paragraph{Test accuracy captures the hyperplane offset.}
Let $D(t):=|b(t)|/\|\w(t)\|$ be the distance of the learned hyperplane from the origin, and let $s$ be the margin of a test point. Once the direction has converged, a point of the class $y=-\operatorname{sign}b(t)$ is misclassified essentially when $s<D(t)$, and points of the other class are always correct. Half of $\mathcal{P}_{conc}$ belongs to the first class, and its margins are approximately uniform on $[\gamma/2,\alpha\gamma]$. Therefore 
\begin{equation}\label{eq:readout}
    Q_{conc}(t)\approx1-\frac12\min\left\{1,\ \max\left(0,\ \frac{D(t)-\gamma/2}{\alpha\gamma-\gamma/2}\right)\right\}.
\end{equation}
Thus $Q_{conc}=\tfrac12$ whenever $D(t)\ge\alpha\gamma$, and $Q_{conc}\ge1-\epsilon$ if and only if $D(t)\le\frac{\gamma}{2}(1+2\epsilon(2\alpha-1))$, which is the threshold in \Cref{thm:main_3}.

Under $\mathcal{P}$, the misclassified points form a strip of width $D(t)$ along the boundary, so the error is only about $0.14\,D(t)$ in $d=2$. This explains \Cref{fig:delayed convergence}. On $\mathcal{P}$ the test accuracy reaches $1$ within about $10^2$ epochs (\Cref{fig:accuracy_same_as_train_dim_2}). On $\mathcal{P}_{conc}$ it rises in Phase~1, dips in Phase~2, and recovers only after about $10^5$ epochs (\Cref{fig:accuracy_dist_shift_dim_2}).\\

\begin{figure*}[!hbt]
    \centering

    \begin{subfigure}[t]{0.29\textwidth}
        \centering
        \includegraphics[width=\linewidth]{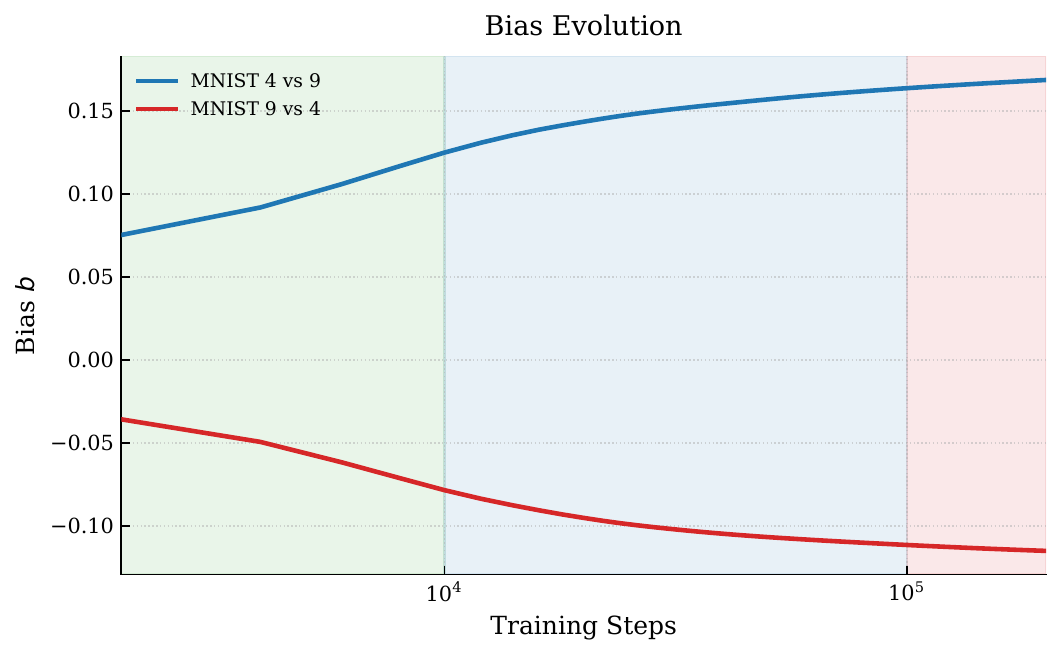}
        \caption{Bias evolution $b(t)$.}
        \label{fig:mnist_bias}
    \end{subfigure}
    \hfill
    \begin{subfigure}[t]{0.29\textwidth}
        \centering
        \includegraphics[width=\linewidth]{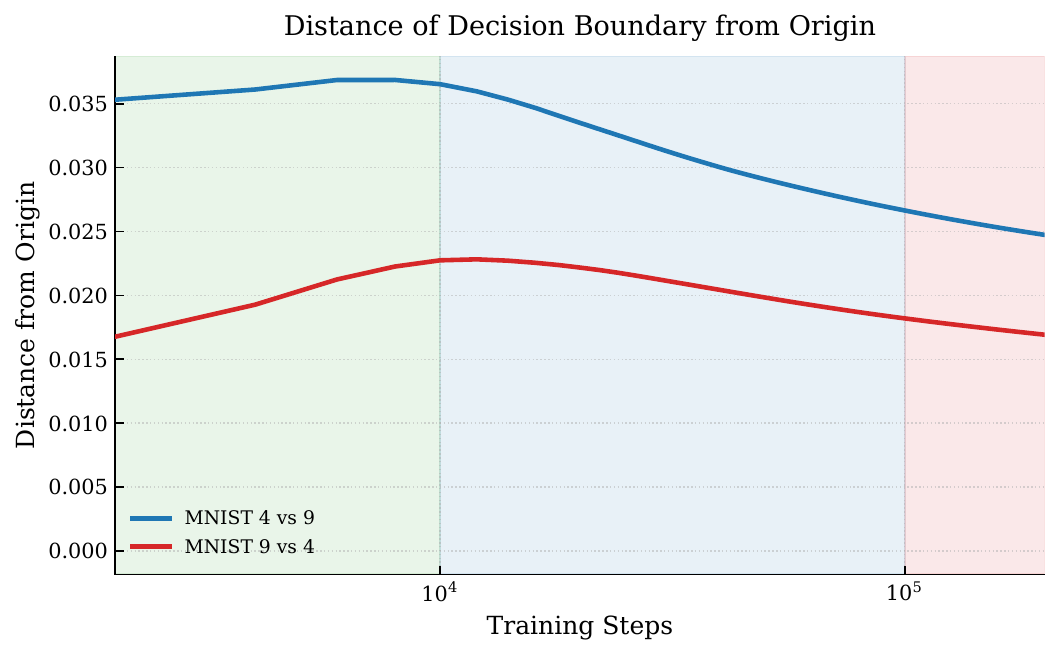}
        \caption{Distance of the separating hyperplane from the origin, $\lvert b(t)\rvert/\|\w(t)\|$.}
        \label{fig:mnist_dfo}
    \end{subfigure}
    \hfill
    \begin{subfigure}[t]{0.32\textwidth}
        \centering
        \includegraphics[width=\linewidth]{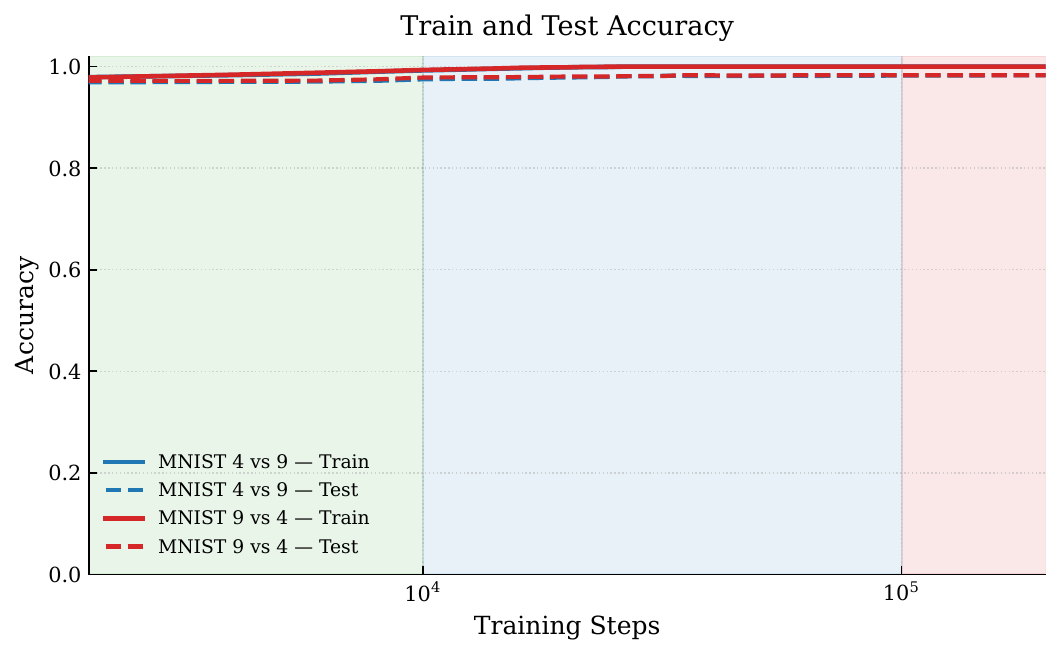}
        \caption{Training and  test accuracies.}
        \label{fig:mnist_acc}
    \end{subfigure}
    \vspace{2\baselineskip}
   \caption{Evolution of (a) the bias $b(t)$ (left), (b) the distance of the separating hyperplane from the origin, $\lvert b(t)\rvert/\|\w(t)\|$ (middle), and (c) training and test accuracies (right), as functions of training epochs, for the MNIST binary classification tasks 4 vs.\ 9 and 9 vs.\ 4. For 4 vs.\ 9, digit 4 is the positive class with $N^+=5743$ and $\lvert \S^+\rvert=195$, while digit 9 is the negative class with $N^-=5862$ and $\lvert \S^-\rvert=178$. For 9 vs.\ 4, digit 9 is the positive class with $N^+=5860$ and $\lvert \S^+\rvert=173$, while digit 4 is the negative class with $N^-=5742$ and $\lvert \S^-\rvert=189$. The shaded regions correspond to phase 1, phase 2, and phase 3, respectively.}
    
    \label{fig:Mnist Experiment}
\end{figure*}

\paragraph{Three phases and the late-time bias.}
With balanced populations (\Cref{fig:4} and the top row of \Cref{fig:symm_asymm+_asymm-}), $b^\star(t)\approx0$, so $b(t)$ stays near $0$ during Phase~1 by \Cref{eq:biasdiffeq}. Phase~1 ends near $10^2$ epochs in all runs, consistent with the polynomial rates in \Cref{thm:rho_t_convergence}, \Cref{lem:rhotconvuse} and \Cref{prop:svdom}(a).

In Phase~2, imbalanced support vectors drive $b(t)$ away from $0$. $D(t)$ then peaks above $\alpha\gamma$ and $Q_{conc}$ collapses to $\tfrac12$; this is the unlearning phase. In Phase~3, $D(t)$ decays as $O(1/\log t)$ and $Q_{conc}$ recovers only after $10^6$ epochs. For $S(350,350)$, $b_\infty\approx0$ and there is no delay.

Since the planted support vectors contribute roughly equally to $\tilde{A}_S^\pm$, \Cref{thm:main_1} predicts $b_\infty=-\log\delta\approx\frac12\log(|\mathcal{S}^+|/|\mathcal{S}^-|)$. This gives $0.144$, $0.092$, $0.049$ and $0$ for the four configurations of \Cref{fig:4}, against measured values of about $0.14$, $0.085$, $0.045$ and $0$ at $10^7$ epochs.

With imbalanced populations (middle and bottom rows of \Cref{fig:symm_asymm+_asymm-}), $b(t)$ first tracks the population value $b^\star(t)$. It then converges to the same support-vector values ($\approx\pm0.14$ or $0$) and changes sign when the two imbalances disagree, as predicted by \Cref{prop:svdom}(b) and \Cref{thm:main_1}.

\paragraph{Generalization time.}
Because $|b(t)|\approx|\log\delta|$ while $\|\w(t)\|$ grows only logarithmically, the time for $D(t)$ to reach the threshold in \Cref{eq:readout} is exponential in $|\log\delta|/(1+2\epsilon(2\alpha-1))$ (\Cref{thm:main_3}). Accordingly, support-vector imbalance delays the recovery of $Q_{conc}$ by two to three orders of magnitude, and the delay grows with $|\log\delta|$ (\Cref{fig:4}). 
A population imbalance adds a transient delay while $b(t)$ relaxes from $b^\star$ at the power-law rate implied by \Cref{eq:biasevol}. The bound is conservative, since $\|\w(t)\|$ remains of order one over our horizon.

\paragraph{MNIST.}
For $f(\x)=\W^\top \h(\x)+\b$, where $\h(\x)$ denotes the hidden-layer representation and $\W$ and $\b$ are the weight vector and bias of the output neuron, respectively, the bias gradient is exactly \Cref{eq:biasdiffeq}, so \Cref{prop:svdom} applies to the output layer with the hidden features held fixed. In both labellings the population and support-vector imbalances point in opposite directions, and the bias follows the support vectors. It is positive ($\approx0.16$) for 4 vs 9 and negative ($\approx-0.11$) for 9 vs 4, since digit 4 has more support vectors in both runs (\Cref{fig:mnist_bias}). $D(t)$ rises until about $10^4$ steps and then decreases (\Cref{fig:mnist_dfo}), the signature of Phases~2--3. 
Standard test accuracy stays near $1$ (\Cref{fig:mnist_acc}), as in \Cref{fig:accuracy_same_as_train_dim_2}. The bias mechanism is therefore present in deep networks.

\section{Conclusion and Limitations}\label{sec:conclusion}

We showed that delayed generalization does not need depth or compositional structure. It already appears in a linear model with a learnable bias, trained by full-batch gradient descent on separable data. The residual $\brho(t)$ converges to $\tilde{\w}$ at rate $O(t^{-\min(\lambda,1,\theta-1)})$ (\Cref{thm:rho_t_convergence}), which is enough for the support vectors to take over (\Cref{lem:rhotconvuse}). The bias then converges to $-\log\delta$, which is set by the support-vector imbalance (\Cref{thm:main_1}). Learning therefore passes through three phases: population dominated, unlearning and generalization. Because $\|\w(t)\|$ grows only logarithmically, the generalization time is at least exponential in $|\log\delta|$ (\Cref{thm:main_3}). Synthetic experiments confirm the phases, the bias limits and the delay. The same bias behavior appears at the output neuron of a network trained on MNIST.

Our setting is deliberately simple. We assume separable data, a max-margin separator through the origin, the exponential loss and full-batch gradient descent. Non-separable data, other losses and weight decay are left open. The delay also depends on the test distribution. It is large when test points lie near the boundary and small otherwise. We do not observe grokking in its usual form. Finally, on MNIST we only track the output bias, and our theory does not describe how hidden features evolve. Connecting the bias mechanism to feature learning is a natural next step.

\section{AI Use Statement}

We used LLMs to improve the language of this paper, including grammar, phrasing and readability, even for proofs. We also used them to draft the Conclusion and Limitations sections from our own description of the results. The figure captions were generated with AI after we provided the context of each figure. We reviewed and edited all AI-generated text. The research questions, theoretical results, proofs, experiments and their interpretation are the authors' own work. The authors take full responsibility for the content of this paper.
\bibliographystyle{apalike}
\bibliography{reference}

@inproceedings{neyshabur,
author = {Neyshabur, Behnam and Bhojanapalli, Srinadh and McAllester, David and Srebro, Nathan},
title = {Exploring generalization in deep learning},
year = {2017},
isbn = {9781510860964},
publisher = {Curran Associates Inc.},
address = {Red Hook, NY, USA},
booktitle = {Proceedings of the 31st International Conference on Neural Information Processing Systems},
pages = {5949–5958},
numpages = {10},
location = {Long Beach, California, USA},
series = {NIPS'17}
}

@inproceedings{
zhang,
title={Understanding deep learning requires rethinking generalization},
author={Chiyuan Zhang and Samy Bengio and Moritz Hardt and Benjamin Recht and Oriol Vinyals},
booktitle={International Conference on Learning Representations},
year={2017},
url={https://openreview.net/forum?id=Sy8gdB9xx}
}

@inproceedings{nanda,
title={Progress measures for grokking via mechanistic interpretability},
author={Neel Nanda and Lawrence Chan and Tom Lieberum and Jess Smith and Jacob Steinhardt},
booktitle={The Eleventh International Conference on Learning Representations },
year={2023},
url={https://openreview.net/forum?id=9XFSbDPmdW}
}

@misc{tan,
title={Understanding Grokking: Insights from Neural Network Robustness},
author={Zhiquan Tan and Weiran Huang},
year={2025},
url={https://openreview.net/forum?id=IRjT0AmsDI}
}

@article{fan,
  author       = {Simin Fan and
                  Razvan Pascanu and
                  Martin Jaggi},
  title        = {Deep Grokking: Would Deep Neural Networks Generalize Better?},
  journal      = {CoRR},
  volume       = {abs/2405.19454},
  year         = {2024},
  url          = {https://doi.org/10.48550/arXiv.2405.19454},
  doi          = {10.48550/ARXIV.2405.19454},
  eprinttype   = {arXiv},
  eprint       = {2405.19454},
  bibsource    = {dblp computer science bibliography, https://dblp.org}
}

@inproceedings{lyu,

title={Dichotomy of Early and Late Phase Implicit Biases Can Provably Induce Grokking},
author={Kaifeng Lyu and Jikai Jin and Zhiyuan Li and Simon Shaolei Du and Jason D. Lee and Wei Hu},
booktitle={The Twelfth International Conference on Learning Representations},
year={2024},
url={https://openreview.net/forum?id=XsHqr9dEGH}
}

@inproceedings{mohamadi,
author = {Mohamadi, Mohamad Amin and Li, Zhiyuan and Wu, Lei and Sutherland, Danica J.},
title = {Why do you grok? a theoretical analysis on grokking modular addition},
year = {2024},
publisher = {JMLR.org},
booktitle = {Proceedings of the 41st International Conference on Machine Learning},
articleno = {1465},
numpages = {34},
location = {Vienna, Austria},
series = {ICML'24}
}

@inproceedings{doshi,

title={To Grok or not to Grok: Disentangling Generalization and Memorization on Corrupted Algorithmic Datasets},
author={Darshil Doshi and Aritra Das and Tianyu He and Andrey Gromov},
booktitle={The Twelfth International Conference on Learning Representations},
year={2024},
url={https://openreview.net/forum?id=UHjE5v5MB7}
}

@article{power,
  author       = {Alethea Power and
                  Yuri Burda and
                  Harri Edwards and
                  Igor Babuschkin and
                  Vedant Misra},
  title        = {Grokking: Generalization Beyond Overfitting on Small Algorithmic Datasets},
  journal      = {CoRR},
  volume       = {abs/2201.02177},
  year         = {2022},
  url          = {https://arxiv.org/abs/2201.02177},
  eprinttype    = {arXiv},
  eprint       = {2201.02177},
  bibsource    = {dblp computer science bibliography, https://dblp.org}
}

@inproceedings{humayun,

title={Deep Networks Always Grok and Here is Why},
author={Ahmed Imtiaz Humayun and Randall Balestriero and Richard Baraniuk},
booktitle={High-dimensional Learning Dynamics 2024: The Emergence of Structure and Reasoning},
year={2024},
url={https://openreview.net/forum?id=NpufNsg1FP}
}

@article{soudry,
author = {Soudry, Daniel and Hoffer, Elad and Nacson, Mor Shpigel and Gunasekar, Suriya and Srebro, Nathan},
title = {The implicit bias of gradient descent on separable data},
year = {2018},
issue_date = {January 2018},
publisher = {JMLR.org},
volume = {19},
number = {1},
issn = {1532-4435},
journal = {J. Mach. Learn. Res.},
month = jan,
pages = {2822–2878},
numpages = {57}
}

@misc{beck,

title={Grokking at the Edge of Linear Separability},
author={Alon Beck and Noam Itzhak Levi and Yohai Bar-Sinai},
year={2025},
url={https://openreview.net/forum?id=l1raPjOUPA}
}

@inproceedings{levi,

title={Grokking in Linear Estimators -- A Solvable Model that Groks without Understanding},
author={Noam Itzhak Levi and Alon Beck and Yohai Bar-Sinai},
booktitle={The Twelfth International Conference on Learning Representations},
year={2024},
url={https://openreview.net/forum?id=GH2LYb9XV0}
}

@inproceedings{liu2023,

title={Omnigrok: Grokking Beyond Algorithmic Data},
author={Ziming Liu and Eric J Michaud and Max Tegmark},
booktitle={The Eleventh International Conference on Learning Representations },
year={2023},
url={https://openreview.net/forum?id=zDiHoIWa0q1}
}

@inproceedings{liu2022,
author = {Liu, Ziming and Kitouni, Ouail and Nolte, Niklas and Michaud, Eric J. and Tegmark, Max and Williams, Mike},
title = {Towards understanding grokking: an effective theory of representation learning},
year = {2022},
isbn = {9781713871088},
publisher = {Curran Associates Inc.},
address = {Red Hook, NY, USA},
booktitle = {Proceedings of the 36th International Conference on Neural Information Processing Systems},
articleno = {2511},
numpages = {13},
location = {New Orleans, LA, USA},
series = {NIPS '22}
}

@article{gromov,
  title={Grokking modular arithmetic},
  author={A. Gromov},
  journal={ArXiv},
  year={2023},
  volume={abs/2301.02679},
  url={https://api.semanticscholar.org/CorpusID:255546130}
}

@article{thilak,
  title={The Slingshot Mechanism: An Empirical Study of Adaptive Optimizers and the Grokking Phenomenon},
  author={Vimal Thilak and Etai Littwin and Shuangfei Zhai and Omid Saremi and Roni Paiss and Joshua M. Susskind},
  journal={ArXiv},
  year={2022},
  volume={abs/2206.04817},
  url={https://api.semanticscholar.org/CorpusID:249605545}
}

@inproceedings{kumar,

title={Grokking as the transition from lazy to rich training dynamics},
author={Tanishq Kumar and Blake Bordelon and Samuel J. Gershman and Cengiz Pehlevan},
booktitle={The Twelfth International Conference on Learning Representations},
year={2024},
url={https://openreview.net/forum?id=vt5mnLVIVo}
}

@article{zunkovic,
author = {{\v Z}unkovi{\v c}, Bojan and Ilievski, Enej},
title = {Grokking phase transitions in learning local rules with gradient descent},
year = {2024},
issue_date = {January 2024},
publisher = {JMLR.org},
volume = {25},
number = {1},
issn = {1532-4435},
journal = {J. Mach. Learn. Res.},
month = jan,
articleno = {199},
numpages = {52}
}

@misc{davies,
title={Delayed Generalization: Bridging Double Descent and Grokking},
author={Xander Davies and Jesse Hoogland and Lauro Langosco and David Krueger},
year={2024},
url={https://openreview.net/forum?id=UtB1Uk1LFE}
}

@inproceedings{nakkiran,

title={Deep Double Descent: Where Bigger Models and More Data Hurt},
author={Preetum Nakkiran and Gal Kaplun and Yamini Bansal and Tristan Yang and Boaz Barak and Ilya Sutskever},
booktitle={International Conference on Learning Representations},
year={2020},
url={https://openreview.net/forum?id=B1g5sA4twr}
}

@article{manningcoe,
doi = {10.1088/2632-2153/ae67d0},
url = {https://doi.org/10.1088/2632-2153/ae67d0},
year = {2026},
month = {jun},
publisher = {IOP Publishing},
volume = {7},
number = {3},
pages = {035045},
author = {Manning-Coe, Dmitry and Gliozzi, Jacopo and Stapleton, Alexander G and Hirst, Edward and De Tomasi, Giuseppe and Bradlyn, Barry and Berman, David},
title = {Grokking vs learning: same features, different encodings},
journal = {Machine Learning: Science and Technology}
}

@article{Papyan_2020,
   title={Prevalence of neural collapse during the terminal phase of deep learning training},
   volume={117},
   ISSN={1091-6490},
   url={http://dx.doi.org/10.1073/pnas.2015509117},
   DOI={10.1073/pnas.2015509117},
   number={40},
   journal={Proceedings of the National Academy of Sciences},
   publisher={Proceedings of the National Academy of Sciences},
   author={Papyan, Vardan and Han, X. Y. and Donoho, David L.},
   year={2020},
   month=sep, pages={24652–24663} }

@article{maa,
author = {Ma, Siyuan and Bassily, Raef and Belkin, Mikhail},
year = {2017},
month = {12},
pages = {},
title = {The Power of Interpolation: Understanding the Effectiveness of SGD in Modern Over-parametrized Learning},
doi = {10.48550/arXiv.1712.06559}
}

@InProceedings{Junior,
  title = 	 {Grokking Beyond the {E}uclidean Norm of Model Parameters},
  author =       {Junior, Tikeng Notsawo Pascal and Dumas, Guillaume and Rabusseau, Guillaume},
  booktitle = 	 {Proceedings of the 42nd International Conference on Machine Learning},
  pages = 	 {28552--28618},
  year = 	 {2025},
  editor = 	 {Singh, Aarti and Fazel, Maryam and Hsu, Daniel and Lacoste-Julien, Simon and Berkenkamp, Felix and Maharaj, Tegan and Wagstaff, Kiri and Zhu, Jerry},
  volume = 	 {267},
  series = 	 {Proceedings of Machine Learning Research},
  month = 	 {13--19 Jul},
  publisher =    {PMLR},
  url = 	 {https://proceedings.mlr.press/v267/junior25a.html}
}

@article{DeMoss,
   title={The complexity dynamics of grokking},
   volume={482},
   ISSN={0167-2789},
   url={http://dx.doi.org/10.1016/j.physd.2025.134859},
   DOI={10.1016/j.physd.2025.134859},
   journal={Physica D: Nonlinear Phenomena},
   publisher={Elsevier BV},
   author={DeMoss, Branton and Sapora, Silvia and Foerster, Jakob and Hawes, Nick and Posner, Ingmar},
   year={2025},
   month=nov, pages={134859} }

@misc{kaplan,
      title={Scaling Laws for Neural Language Models}, 
      author={Jared Kaplan and Sam McCandlish and Tom Henighan and Tom B. Brown and Benjamin Chess and Rewon Child and Scott Gray and Alec Radford and Jeffrey Wu and Dario Amodei},
      year={2020},
      eprint={2001.08361},
      archivePrefix={arXiv},
      primaryClass={cs.LG},
      url={https://arxiv.org/abs/2001.08361}, 
}

\section*{Checklist}

\begin{enumerate}

  \item For all models and algorithms presented, check if you include:
  \begin{enumerate}
    \item A clear description of the mathematical setting, assumptions, algorithm, and/or model. [Yes. Please refer to \Cref{sec:model}.]
    \item An analysis of the properties and complexity (time, space, sample size) of any algorithm. [Not Applicable.]
    \item (Optional) Anonymized source code, with specification of all dependencies, including external libraries. [Yes. We provide this in the supplementary material.]
  \end{enumerate}

  \item For any theoretical claim, check if you include:
  \begin{enumerate}
    \item Statements of the full set of assumptions of all theoretical results. [Yes.]
    \item Complete proofs of all theoretical results. [Yes. Please refer to \Cref{app:sec:proofs}.]
    \item Clear explanations of any assumptions. [Yes.]
  \end{enumerate}

  \item For all figures and tables that present empirical results, check if you include:
  \begin{enumerate}
    \item The code, data, and instructions needed to reproduce the main experimental results (either in the supplemental material or as a URL). [Yes. Please refer to \Cref{app:experimental_setup}.]
    \item All the training details (e.g., data splits, hyperparameters, how they were chosen). [Yes. Please refer to \Cref{app:sec:hyper}.]
    \item A clear definition of the specific measure or statistics and error bars (e.g., with respect to the random seed after running experiments multiple times). [No.]
    \item A description of the computing infrastructure used. (e.g., type of GPUs, internal cluster, or cloud provider). [Yes. Please refer to \Cref{app:sec:compute}.]
  \end{enumerate}

  \item If you are using existing assets (e.g., code, data, models) or curating/releasing new assets, check if you include:
  \begin{enumerate}
    \item Citations of the creator If your work uses existing assets. [Not Applicable.]
    \item The license information of the assets, if applicable. [Not Applicable.]
    \item New assets either in the supplemental material or as a URL, if applicable. [Not Applicable]
    \item Information about consent from data providers/curators. [Not Applicable.]
    \item Discussion of sensible content if applicable, e.g., personally identifiable information or offensive content. [Not Applicable.]
  \end{enumerate}

  \item If you used crowdsourcing or conducted research with human subjects, check if you include:
  \begin{enumerate}
    \item The full text of instructions given to participants and screenshots. [Not Applicable.]
    \item Descriptions of potential participant risks, with links to Institutional Review Board (IRB) approvals if applicable. [Not Applicable.]
    \item The estimated hourly wage paid to participants and the total amount spent on participant compensation. [Not Applicable.]
  \end{enumerate}

\end{enumerate}

\appendix
\thispagestyle{empty}

\onecolumn
\aistatstitle{Supplementary Materials}

\section{Notation}\label{app:notation}

A comprehensive list of the notation used throughout the paper can be found below for reference.
\begin{itemize}
    \item[i.] $\mathcal{D}=\{(\x_i,y_i)\}_{i=1}^N$: the training dataset, with features $\x_i\in\R^d$ and labels $y_i\in\{-1,1\}$.
    \item[ii.] $N$: the number of training points.
    \item[iii.] $N_+$ and $N_-$: the sizes of the positive and negative classes.
    \item[iv.] $I^+$ and $I^-$: the index sets of the positive and negative training points.
    \item[v.] $\mathcal{P}$: the training distribution, also used as the standard test distribution.
    \item[vi.] $\mathcal{P}_{conc}$: the concentrated test distribution, whose margins lie in $[\gamma/2,\alpha\gamma]$.
    \item[vii.] $\mathcal{P}_{test}\in\{\mathcal{P},\mathcal{P}_{conc}\}$: the test distribution.
    \item[viii.] $\gamma$: the margin parameter; every training point lies at distance at least $\gamma/2$ from the separating hyperplane (\Cref{ass:sepmargin}).
    \item[ix.] $\w_*$: a unit-norm separator through the origin attaining the margin $\gamma/2$ (\Cref{ass:sepmargin}).
    \item[x.] $\hat{\w}$: the maximum-margin vector (\Cref{eq:diffsvm}), with $\|\hat{\w}\|=2/\gamma$.
    \item[xi.] $\S$: the set of support vectors, i.e., the training points with $y_i\hat{\w}^\top\x_i=1$.
    \item[xii.] $\S^+$ and $\S^-$: the positive and negative support vectors.
    \item[xiii.] $\alpha_i>0$: the dual coefficients of the support vectors, $\hat{\w}=\sum_{i=1}^k\alpha_iy_{s_i}\x_{s_i}$.
    \item[xiv.] $R$: the bound on the norm of every data point, $\|\x\|\le R$ for all training and test points (\Cref{ass:support}).
    \item[xv.] $\hat{y}(\x;\w,b)=\w^\top\x+b$: the prediction of the linear model.
    \item[xvi.] $\L(\w,b)=\sum_{i=1}^N\exp(-y_i(\w^\top\x_i+b))$: the exponential loss.
    \item[xvii.] $\eta$: the constant step size of gradient descent.
    \item[xviii.] $\w(t)$ and $b(t)$: the weights and the bias after $t$ steps of gradient descent.
    \item[xix.] $P(t)$: the training accuracy at step $t$.
    \item[xx.] $Q(t)$: the test accuracy at step $t$ under $\mathcal{P}_{test}$.
    \item[xxi.] $T_A(\epsilon)$: the first time after which $A$ stays above $(1-\epsilon)\sup_uA(u)$ (\Cref{def:delgen}).
    \item[xxii.] $T_P(\epsilon)$: the training time.
    \item[xxiii.] $T_Q(\epsilon)$: the generalization time.
    \item[xxiv.] $\brho(t)=\w(t)-\hat{\w}\log t$: the residual (\Cref{eq:asymptotic_decomposition}).
    \item[xxv.] $B$: the bound on the residual, $\|\brho(t)\|\le B$ for all $t$.
    \item[xxvi.] $\tilde{\w}=\lim_{t\to\infty}\brho(t)$: the limit of the residual, satisfying $\eta\exp(-\x_n^\top\tilde{\w})=\alpha_n$ for $n\in\S$ with labels absorbed.
    \item[xxvii.] $\r(t)=\brho(t)-\tilde{\w}$: the deviation of the residual from its limit.
    \item[xxviii.] $\H=\sum_{\x_n\in\S}\alpha_n\x_n\x_n^\top$: the weighted second-moment matrix of the support vectors.
    \item[xxix.] $\lambda$: the minimum eigenvalue of $\H$ on the span of the support vectors.
    \item[xxx.] $\theta=\min_{i\notin\S}y_i\hat{\w}^\top\x_i>1$: the minimum normalized margin of the non-support vectors.
    \item[xxxi.] $p=\min(\lambda,1,\theta-1)$: the convergence rate of the residual (\Cref{thm:rho_t_convergence}).
    \item[xxxii.] $C$ and $T_0$: the constants with $\|\r(t)\|\le Ct^{-p}$ for $t>T_0$ (\Cref{thm:rho_t_convergence}).
    \item[xxxiii.] $A^+(t)=\sum_{i\in I^+}\exp(-\w(t)^\top\x_i)$: the population weight of the positive class.
    \item[xxxiv.] $A^-(t)=\sum_{i\in I^-}\exp(\w(t)^\top\x_i)$: the population weight of the negative class.
    \item[xxxv.] $A_S^+(t)=\sum_{i\in\S^+}\exp(-\w(t)^\top\x_i)$: the restriction of $A^+(t)$ to the support vectors.
    \item[xxxvi.] $A_S^-(t)=\sum_{i\in\S^-}\exp(\w(t)^\top\x_i)$: the restriction of $A^-(t)$ to the support vectors.
    \item[xxxvii.] $\vartheta\in(0,1)$: the threshold defining the phases.
    \item[xxxviii.] $T_1(\vartheta)$: the end of Phase~1, after which the support vectors carry at least a fraction $1-\vartheta$ of $A^\pm(t)$ (\Cref{def:phases}).
    \item[xxxix.] $T_2(\vartheta)$: the end of Phase~2, the first time $|b(t)|/\|\w(t)\|$ attains its supremum over $t\ge T_1(\vartheta)$ (\Cref{def:phases}).
    \item[xl.] $b^*(t)=b^\star(t)=\frac12\log\frac{A^+(t)}{A^-(t)}$: the population equilibrium of the bias (\Cref{lem:popdyn}).
    \item[xli.] $b_S^*(t)=b_S^\star(t)=\frac12\log\frac{A_S^+(t)}{A_S^-(t)}$: the support-vector equilibrium of the bias (\Cref{prop:svdom}).
    \item[xlii.] $t_0(\epsilon)$: the first time after which $\|\brho(s)-\tilde{\w}\|\le\epsilon$ for all $s\ge t_0(\epsilon)$ (\Cref{lem:rhotconvuse}).
    \item[xliii.] $\vartheta(\epsilon)=N\exp(2R(B+\epsilon))t_0(\epsilon)^{1-\theta}$: the threshold reached by the time the residual is $\epsilon$-close to its limit (\Cref{lem:rhotconvuse}).
    \item[xliv.] $\tilde{A}_S^+=\sum_{i\in\S^+}\exp(-\tilde{\w}^\top\x_i)$: the limit of $tA_S^+(t)$.
    \item[xlv.] $\tilde{A}_S^-=\sum_{i\in\S^-}\exp(\tilde{\w}^\top\x_i)$: the limit of $tA_S^-(t)$.
    \item[xlvi.] $\delta=\sqrt{\tilde{A}_S^-/\tilde{A}_S^+}$: the support-vector imbalance.
    \item[xlvii.] $b_\infty=-\log\delta$: the limit of the bias (\Cref{thm:main_1}).
    \item[xlviii.] $K$: the bound on the bias, $|b(t)|\le K$ for all $t\ge1$.
    \item[xlix.] $C$ and $p=\min(\lambda,1,\theta-1)$: the constants with $\|\brho(t)-\tilde{\w}\|\le kt^{-p}$ for $t\ge T_0$ (\Cref{thm:main_1}).
    \item[l.] $\theta^+=\min_{i\in I^+\setminus\S^+}\hat{\w}^\top\x_i>1$: the minimum margin of the positive non-support vectors.
    \item[li.] $\theta^-=\max_{i\in I^-\setminus\S^-}\hat{\w}^\top\x_i<-1$: the maximum of $\hat{\w}^\top\x_i$ over the negative non-support vectors.
    \item[lii.] $C_2=|I^+\setminus\S^+|\exp(RB)$: the bound on the contribution of the positive non-support vectors.
    \item[liii.] $C_3=|I^-\setminus\S^-|\exp(RB)$: the bound on the contribution of the negative non-support vectors.
    \item[liv.] $t_0\ge T_0$: the time from which the envelopes in \Cref{thm:main_1} are run.
    \item[lv.] $\bar A_{\S}^+=\exp(kRt_0^{-p}+C_2t_0^{1-\theta^+}/\tilde{A}_S^+)\tilde{A}_S^+$: the positive coefficient of the upper envelope.
    \item[lvi.] $\bar A_{\S}^-=\exp(-kRt_0^{-p})\tilde{A}_S^-$: the negative coefficient of the upper envelope.
    \item[lvii.] $\underline{A}_{\S}^+=\exp(-kRt_0^{-p})\tilde{A}_S^+$: the positive coefficient of the lower envelope.
    \item[lviii.] $\underline{A}_{\S}^-=\exp(kRt_0^{-p}+C_3t_0^{1+\theta^-}/\tilde{A}_S^-)\tilde{A}_S^-$: the negative coefficient of the lower envelope.
    \item[lix.] $\bar\delta=\sqrt{\bar A_{\S}^-/\bar A_{\S}^+}$: the imbalance of the upper envelope, whose limit is $-\log\bar\delta$.
    \item[lx.] $\underline\delta=\sqrt{\underline{A}_{\S}^-/\underline{A}_{\S}^+}$: the imbalance of the lower envelope, whose limit is $-\log\underline\delta$.
    \item[lxi.] $M=\max(K,|\log\bar\delta|,|\log\underline\delta|)$: the range of the bias and the envelopes in \Cref{thm:main_1}. In \Cref{app:lem:rhotconvuse_max}, $M$ instead denotes the ratio of non-support-vector to support-vector weight.
    \item[lxii.] $\bar b(t)$: the upper envelope of the bias (\Cref{eq:biasdyn}).
    \item[lxiii.] $\underline b(t)$: the lower envelope of the bias (\Cref{eq:biasdyn}).
    \item[lxiv.] $\phi(t,b)$ and $\psi(t,b)$: the positive and negative parts of the bias update in the proof of \Cref{app:thm:main_1}.
    \item[lxv.] $\alpha$: the sensitivity parameter of $\mathcal{P}_{conc}$ (distinct from the dual coefficients $\alpha_i$).
    \item[lxvi.] $s=y\hat{\w}^\top\x/\|\hat{\w}\|$: the margin of a test point $(\x,y)$.
    \item[lxvii.] $\epsilon$: the accuracy tolerance in $T_P(\epsilon)$ and $T_Q(\epsilon)$.
    \item[lxviii.] $D(t)=|b(t)|/\|\w(t)\|$: the distance of the learned hyperplane from the origin.
    \item[lxix.] $\w_{\text{true}}=\frac{1}{\sqrt d}[1,\ldots,1]^\top$: the ground-truth direction in the synthetic experiments.
\end{itemize}

\section{Proofs}\label{app:sec:proofs}

\subsection{SVM with a bias}
\begin{theorem}\label{app:thm:svm}
    Under \Cref{ass:sepmargin} the SVM solution given by
    \begin{align*}
    \w_{SVM} = \arg \min ||\w||^2  \text{   s.t.   }y_i(\w^\top\x+b)\geq 1 
\end{align*}
is equal to
\begin{align*}
    \hat{\w} = \arg \min ||\w||^2 + b^2  \text{   s.t.   } y_i(\w^\top\x+b)\geq 1.
\end{align*}
\end{theorem}

\begin{proof}
    From \Cref{ass:sepmargin}, the solution to \Cref{eq:svm} is some $(\w_{SVM}, 0)$. Hence for all feasible $(\w, b)$ from the optimality of $(\w_{SVM}, 0)$, we have $||\w||^2 \geq ||\w_{SVM}||^2$. Note that both problems \Cref{eq:diffsvm} and \Cref{eq:svm} have the same feasible region. Hence for all feasible $(\w, b)$ for \Cref{eq:diffsvm}, we also have $||\w||^2+b^2 \geq ||\w||^2 \geq  ||\w_{SVM}||^2+0^2$ implying that $(\w_{SVM}, 0)$ also optimizes \Cref{eq:diffsvm}. Since \Cref{eq:diffsvm} is a strictly convex optimization problem it has a unique solution. Hence $\hat{\w} = \w_{SVM}$.
\end{proof}

\subsection{Early time bias dynamics}

\begin{lemma}\label{app:lem:popdyn}
Suppose $b^*(t) = \frac{1}{2}\log \frac{A^+(t)}{A^-(t)}$. For every $t \in \mathbb{N}_0$, gradient descent on the bias gives
    \begin{align*}
        b(t+1) - b(t) &= \eta(A^+(t)\exp(-b(t)) - A^-(t)\exp(b(t)))\\
        &= -2\eta\sqrt{A^+(t)A^-(t)} \sinh(b(t) - b^*(t)).
    \end{align*}
\end{lemma}

\begin{proof}
The gradient of the loss with respect to the bias gives 
\begin{align*}
    \frac{\partial \mathcal{L}}{\partial b}
= -\sum_n y_n \exp({-y_n(\mathbf{w}^{\top}\mathbf{x}_n+b)})
\end{align*}
and hence
\begin{align*}
b(t+1)-b(t)
= \eta\sum_n y_n \exp({-y_n(\mathbf{w}(t)^{\top}\mathbf{x}_n+b(t))})
= \eta\left(A^+(t)\exp({-b})-A^-(t)\exp(b)\right).
\end{align*}
With $\exp(b^*)=\sqrt{\frac{A^+}{A^-}}$, we can multiply both sides by $\sqrt{A^+A^-}$ to obtain  $A^+=\sqrt{A^+A^-}\exp(b^*)$. Similarly we have $A^-=\sqrt{A^+A^-}\exp(-b^*)$. Combining, we get $\eta\left(A^+\exp(-b)-A^-\exp(b)\right)
=-2\eta\sqrt{A^+(t)A^-(t)}\sinh(b-b^*)$. This completes the proof.
\end{proof}

\subsection{Sufficient condition for onset of Phase 2}
\begin{lemma}\label{app:lem:rhotconvuse_max}
Let $\vartheta\in(0,1)$ and $\varepsilon>0$, and suppose $||\brho(t)-\tilde{\w}||\le\varepsilon$ for all $t\ge t_0\ge1$. Define the following quantities.
\begin{itemize}
    \item $\tilde{A}_S^+=\sum_{i\in\S^+}\exp(-\tilde{\w}^\top\x_i)$ and $\tilde{A}_S^-=\sum_{i\in\S^-}\exp(\tilde{\w}^\top\x_i)$.
    \item $M:=\max\Big\{\frac{1}{\tilde{A}_S^+}\sum_{i\in I^+\setminus\S^+}\exp({-\tilde{\w}^\top\x_i}),\ \frac{1}{\tilde{A}_S^-}\sum_{i\in I^-\setminus\S^-}\exp({\tilde{\w}^\top\x_i})\Big\}$.
\end{itemize}
Then
\begin{align*}
T_1(\vartheta)\ \le\ \max\Big\{t_0,\ \Big(\frac{(1-\vartheta)\,\exp(2R\varepsilon)M}{\vartheta}\Big)^{\frac1{\theta-1}}\Big\}.
\end{align*}
\end{lemma}

\begin{proof}
Fix $t\ge t_0$. By \Cref{eq:asymptotic_decomposition}, $\w(t)=\hat{\w}\log t+\brho(t)$. For $i\in I^+$ this gives us $\exp({-\w(t)^\top\x_i})=t^{-y_i\hat{\w}^\top\x_i}\exp({-\brho(t)^\top\x_i})$. For $i\in I^-$, this gives us $\exp({\w(t)^\top\x_i})=t^{-y_i\hat{\w}^\top\x_i}\exp({\brho(t)^\top\x_i})$.

Since $|\brho(t)^\top\x_i-\tilde{\w}^\top\x_i|\le R\varepsilon$, we have $\exp(\mp\brho(t)^\top\x_i)\in \exp(\mp\tilde{\w}^\top\x_i)[\exp(-R\varepsilon),\exp(R\varepsilon)]$. Recall $y_i\hat{\w}^\top\x_i=1$ on $\S$ and $y_i\hat{\w}^\top\x_i\ge\theta$ off $\S$. Hence
\begin{align*}
    A_S^\pm(t)\in t^{-1}\tilde{A}_S^\pm\,[\exp(-R\varepsilon),\exp(R\varepsilon)],
\end{align*}and
\begin{align*}
    A^\pm(t)-A_S^\pm(t)\le t^{-\theta}\exp({R\varepsilon})\sum_{i\in I^\pm\setminus\S^\pm}\exp({\mp\tilde{\w}^\top\x_i}).
\end{align*}
Let $r_\pm(t):=\big(A^\pm(t)-A_S^\pm(t)\big)/A_S^\pm(t)$ be the ratio of non-support-vector to support-vector weight. The ratio $A_S^\pm(t)/A^\pm(t)=1/(1+r_\pm(t))$ is at least $1-\vartheta$ if and only if $r_\pm(t)\le\vartheta/(1-\vartheta)$.

We have $r_\pm(t)\le \exp(2R\varepsilon)M\,t^{1-\theta}$. This is at most $\vartheta/(1-\vartheta)$ for every $t\ge\big((1-\vartheta)\exp(2R\varepsilon)M/\vartheta\big)^{1/(\theta-1)}$. So both are at least $1-\vartheta$ for every $t$ beyond the stated maximum, which bounds $T_1(\vartheta)$.
\end{proof}

\begin{corollary}\label{app:cor:T1suff}
Let $\varepsilon>0$ and let $B\ge||\brho(t)||$ for all $t\ge1$. Define the following quantities.
\begin{itemize}
    \item $t_0(\varepsilon):=\min\big\{t\ge1:\ ||\brho(s)-\tilde{\w}||\le\varepsilon\ \ \forall s\ge t\big\}$, the first time after which the residual stays $\varepsilon$-close to its limit.
    \item $\vartheta(\varepsilon):=N\,\exp(2R(B+\varepsilon))\,t_0(\varepsilon)^{1-\theta}$.
\end{itemize}
If $\vartheta(\varepsilon)<1$, then $T_1\big(\vartheta(\varepsilon)\big)\ \le\ t_0(\varepsilon)$. That is, by the time the residual has converged to within $\varepsilon$, the support vectors carry at least a fraction $1-\vartheta(\varepsilon)$ of both $A^+(t)$ and $A^-(t)$, and Phase~2 has begun.
\end{corollary}

\begin{proof}
By \Cref{app:lem:rhotconvuse_max}, applied with $t_0=t_0(\varepsilon)$,
\begin{align*}
T_1(\vartheta)\le\max\Big\{t_0(\varepsilon),\ \Big(\frac{(1-\vartheta)\,\exp(2R\varepsilon)M}{\vartheta}\Big)^{\frac1{\theta-1}}\Big\}\qquad\text{for every }\vartheta\in(0,1).
\end{align*}
We first bound $M$. Since $||\brho(t)||\leq B$ for all $t$ and $\brho(t)\to\tilde{\w}$, also $||\tilde{\w}||\leq B$. Hence $\exp(\mp\tilde{\w}^\top\x_i)\in[\exp(-RB),\exp(RB)]$ for every $i$. Each class contains at least one support vector, so $\tilde{A}_S^\pm\ge \exp(-RB)$. Each sum over non-support vectors has fewer than $N$ terms, so it is at most $N\exp(RB)$. Therefore $M\le N\exp(2RB)$.

Now take $\vartheta=\vartheta(\varepsilon)$. Using $1-\vartheta\le1$ and the bound on $M$,
\begin{align*}
\frac{(1-\vartheta)\,\exp(2R\varepsilon)M}{\vartheta}\ \le\ \frac{N\,\exp(2R(B+\varepsilon))}{\vartheta(\varepsilon)}\ =\ t_0(\varepsilon)^{\theta-1}.
\end{align*}
So the second term of the maximum is at most $t_0(\varepsilon)$, and $T_1(\vartheta(\varepsilon))\le t_0(\varepsilon)$.
\end{proof}

\subsection{Bias dynamics in Phase 2}

\begin{proposition}\label{app:prop:svdom}
Let \Cref{ass:sepmargin} and \Cref{ass:support} hold, and let $C,p$ be as in \Cref{thm:rho_t_convergence}. Define the following quantities.
\begin{itemize}
    \item $b^\star(t):=\frac12\log\frac{A^+(t)}{A^-(t)}$.
    \item $b_S^\star(t):=\frac12\log\frac{A_S^+(t)}{A_S^-(t)}$.
\end{itemize}
Then for all $t\ge1$,
\begin{enumerate}
\item[(a)] $\max\left(1-\frac{A_S^+(t)}{A^+(t)},  1-\frac{A_S^-(t)}{A^-(t)}\right) \leq N \exp(2RB)t^{1-\theta}$.
\item[(b)] If $\min\left(\frac{A_S^+(t)}{A^+(t)},  \frac{A_S^-(t)}{A^-(t)}\right)\ge1-\vartheta$, then $\displaystyle |b^\star(t)-b_S^\star(t)|\le\frac12\log\frac1{1-\vartheta}$.
\end{enumerate}
\end{proposition}

\begin{proof}

\emph{(a)} Consider the class $+1$ without loss of generality. The same argument applies for the $-1$ class. Each support vector term equals $t^{-1}\exp(-\brho(t)^\top\x_i)\ge t^{-1}\exp(-RB)$, so $A_S^+(t)\ \ge\ |\S^+|\,t^{-1}\exp(-RB)$. Each non-support-vector term is at most $t^{-\theta}\exp(RB)$, so $A^+(t)-A_S^+(t)\ \le\ \big(|I^+|-|\S^+|\big)\,t^{-\theta}\exp(RB)$. Using $A^+(t)\ge A_S^+(t)$,
\begin{align*}
1-\frac{A_S^+(t)}{A^+(t)}=\frac{A^+(t)-A_S^+(t)}{A^+(t)}\le\frac{A^+(t)-A_S^+(t)}{A_S^+(t)}\le\frac{|I^+|-|\S^+|}{|\S^+|}\,\exp(2RB)\,t^{1-\theta}\le N\exp(2RB)\,t^{1-\theta}.
\end{align*}
\emph{(b)} By definition,
\begin{align*}
b^\star(t)-b_S^\star(t)=\frac12\Big[\log\frac{A^+(t)}{A_S^+(t)}-\log\frac{A^-(t)}{A_S^-(t)}\Big].
\end{align*}
If both shares are at least $1-\vartheta$, then both ratios $A^\pm(t)/A_S^\pm(t)$ lie in $[1,\frac1{1-\vartheta}]$, so both logarithms lie in $[0,\log\frac1{1-\vartheta}]$. The difference of two numbers in this interval has absolute value at most $\log\frac1{1-\vartheta}$, which gives the claim after multiplying by $\frac12$.
\end{proof}

\subsection{Convergence of the residual}
\begin{theorem}\label{app:thm:rho_t_convergence}
Define the following quantities.
\begin{itemize}
    \item $\H = \sum\limits_{\x_n \in \S} \alpha_n \x_n\x_n^T$.
    \item $\lambda > 0$, the minimum eigenvalue of $\H$ on the span of the support vectors.
    \item $\theta > 1$, the minimum margin $y_n\hat{\w}^\top\x_n$ of the non-support vectors.
    \item $\r(t) = \w(t) -\hat{\w}\log t - \tilde{\w} = \brho(t) - \tilde{\w}$.
\end{itemize}
Assume $\lambda \notin \{1, \theta -1\}$. Then there exist constants $C$ and $T_0$ such that $||\r(t)||\le Ct^{-p}$ for all $t > T_0$, where $p=\min(\lambda, 1, \theta-1)$.
\end{theorem}
\begin{proof}
    Recall that $\w(t) = \hat{\w}\log t + \brho(t)$ where $\brho(t)$ is bounded. Define $\r(t) = \w(t) -\hat{\w}\log t - \tilde{\w} = \brho(t) - \tilde{\w}$. The loss is given by $\L(\w) = \sum_{n=1}^N \exp(-\w^T\x_n)$. Hence, $\nabla \L(\w) = -\sum_{n=1}^N \exp(-\w^T\x_n)\x_n$. Plugging this into a gradient step, we have
    \begin{align*}
        \w(t+1) &= \w(t) + \eta \sum_{n=1}^N \exp(-\w(t)^T\x_n)\x_n\\
        &= \w(t) + \eta \sum_{n=1}^N \exp(-(\hat{\w}\log t + \tilde{\w}+\r(t))^T\x_n)\x_n\\
        &= \w(t) + \eta \sum_{n=1}^N \exp(-\hat{\w}^T\x_n\log t) \exp(- \tilde{\w}^T\x_n)\exp(-\r(t)^T\x_n)\x_n\\
        &= \w(t) + \eta \sum_{n=1}^N t^{-\hat{\w}^T\x_n} \exp(- \tilde{\w}^T\x_n)\exp(-\r(t)^T\x_n)\x_n.
    \end{align*}
    For support vectors $\hat{\w}^T\x_n = 1$ and for non-support vectors, $\hat{\w}^T\x_n \geq \theta > 1$. Write $\hat{\w} = \sum_n \alpha_n \x_n$ for all $\x_n \in \S$. Using \Cref{eq:asymptotic_decomposition}, we have $\eta \exp(-\tilde{\w}^T\x_n) = \alpha_n$. Substituting,
    \begin{align*}
        \w(t+1) &= \w(t) + \eta \sum_{\x_n \in \S} \frac{1}{t} \exp(- \tilde{\w}^T\x_n)\exp(-\r(t)^T\x_n)\x_n  + \eta \sum_{\x_n \notin \S} t^{-\hat{\w}^T\x_n} \exp(- \tilde{\w}^T\x_n)\exp(-\r(t)^T\x_n)\x_n\\
        &= \w(t) + \sum_{\x_n \in \S} \frac{1}{t}\alpha_n \exp(-\r(t)^T\x_n)\x_n + O(t^{-\theta})\\
         &= \w(t) + \sum_{\x_n \in \S} \frac{1}{t}\alpha_n (1-\r(t)^T\x_n + O((\r(t)^T\x_n)^2))\x_n + O(t^{-\theta})\\
         &= \w(t) + \sum_{\x_n \in \S} \frac{1}{t}\alpha_n (1-\r(t)^T\x_n + O(||\r(t)||^2))\x_n + O(t^{-\theta})\\
         &= \w(t) + \sum_{\x_n \in \S} \frac{1}{t}\alpha_n\x_n  - \frac{1}{t}\sum_{\x_n \in \S} \alpha_n (\r(t)^T\x_n) \x_n + O(||\r(t)||^2/t) + O(t^{-\theta})\\
          &= \w(t) + \frac{\hat{\w}}{t}  - \frac{1}{t}\sum_{\x_n \in \S}\alpha_n \x_n\x_n^T \r(t) + O(||\r(t)||^2/t) + O(t^{-\theta}).
    \end{align*}
    In the above, $O(\cdot)$ denotes a vector whose norm is $O(\cdot)$. The third line uses $\exp(-u)=1-u+O(u^2)$ for $|u|\le1$, which applies since $|\r(t)^T\x_n|\le R||\r(t)||\to0$. 
    Since $\r(t) = \w(t) - \hat{\w}\log t - \tilde{\w}$, $\r(t+1)-\r(t) = \w(t+1) - \w(t) - \hat{\w}(\log(t+1)-\log(t))$. Putting everything together,
    \begin{align*}
        \r(t+1)-\r(t) &= \w(t+1) - \w(t) - \hat{\w}(\log(t+1)-\log(t))\\
        &= \frac{\hat{\w}}{t}  - \frac{1}{t}\sum_{\x_n \in \S}\alpha_n \x_n\x_n^T \r(t) + O(||\r(t)||^2/t) + O(t^{-\theta}) - \hat{\w} \log (1+1/t)\\
        &= \frac{\hat{\w}}{t}  - \frac{1}{t}\H\r(t) + O(||\r(t)||^2/t) + O(t^{-\theta}) - \frac{\hat{\w}}{t} + O(t^{-2})\\
        &= -\frac{\H}{t} \r(t) + O(||\r(t)||^2/t) + O(t^{-\theta})  + O(t^{-2}).
    \end{align*}
We wish to control the slowest term. Suppose the slowest term is $t^{-\theta}$. Then, for $t\ge\lambda_{\max}(\H)$, using $||(\I-\H/t)\r||\le(1-\lambda/t)||\r||$ for $\r$ in the span of the support vectors (the component of $\r(t)$ orthogonal to this span changes only by the $O(t^{-\theta})$ term, so since $\r(t)\to0$ it is $O(t^{1-\theta})$),
\begin{align*}
    \r(t+1) &= (\I - \frac{\H}{t})\r(t) + O(t^{-\theta})\\
   \implies ||\r(t+1)|| &\leq (1 - \frac{\lambda}{t})||\r(t)|| + O(t^{-\theta}).
\end{align*}
Define $P_t = \prod_{s=T_0}^{t-1} (1-\lambda/s)$ for a large enough $T_0>\lambda$. Since $\log(1-x)=-x+O(x^2)$ and $\sum_{s=T_0}^{t-1}\frac1s=\log t+O(1)$,
\begin{align*}
    P_t = \prod_{s=T_0}^{t-1} \Big(1-\frac{\lambda}{s}\Big) = \exp\Big(-\lambda \sum_{s=T_0}^{t-1} \frac{1}{s}+O(1)\Big) =  \Theta(t^{-\lambda})
\end{align*}
Then dividing by $P_{t+1}$ on both sides and canceling the $1-\lambda/t$ term off,
\begin{align*}
    \frac{||\r(t+1)||}{P_{t+1}} &\leq \frac{||\r(t)||}{P_t} + \frac{O(t^{-\theta})}{P_{t+1}}\\
    &\leq \frac{||\r(t)||}{P_t} + O(t^{\lambda - \theta}).
\end{align*}
Summing over $t$ from $T_0$, $\frac{||\r(t)||}{P_{t}} \leq \frac{||\r(T_0)||}{P_{T_0}} + O(\sum_{t=T_0}^t t^{\lambda - \theta})$.
If $\lambda - \theta > -1$, then the sum grows as $O(t^{\lambda - \theta + 1})$ and hence, $\frac{||\r(t)||}{P_{t}} \leq O(t^{\lambda - \theta + 1})$. Therefore, $||\r(t)|| \leq O(t^{\lambda - \theta +1-\lambda}) = O(t^{1-\theta})$. 

Otherwise ($\lambda-\theta<-1$), the sum converges, so $\frac{||\r(T_0)||}{P_{T_0}} + O(1) = O(1)$ and we get $||\r(t)|| = O(t^{-\lambda})$. Putting these together, $||\r(t)|| = O(t^{-\min(\lambda, \theta-1)})$.

Suppose $t^{-2}$ is the slowest term. Using the same argument as above, we can conclude that $||\r(t)|| = O(t^{-\min(\lambda, 1)})$. Otherwise, $||\r(t)||^2/t$ is the slowest term.
    
From \Cref{eq:asymptotic_decomposition} we know that $||\r(t)||$ goes to 0. Hence, for any $\varepsilon > 0$, there exists a $T_0$ such that $\forall t > T_0$, $||\r(t)|| \leq \varepsilon$. Choose $\varepsilon$ small enough that the $O(||\r(t)||^2/t)$ term is at most $\lambda||\r(t)||/(2t)$. Then for $t \geq T_0$ (in this case the other terms are of lower order):
    \begin{align*}
        ||\r(t+1)|| &\leq (1 - \frac{\lambda}{t})||\r(t)|| + \frac{\lambda ||\r(t)||}{2t}\\
        &\leq (1-\frac{\lambda}{2t})||\r(t)|| \leq \prod_{s=T_0}^{t} \Big(1-\frac{\lambda}{2s}\Big)||\r(T_0)||\\ 
        &\leq \prod_{s=T_0}^{t} \exp\Big(-\frac{\lambda}{2s}\Big) ||\r(T_0)|| \leq \exp\Big(-\frac{\lambda}{2} \log\frac{t}{T_0}\Big) ||\r(T_0)||\\
        &= \Big(\frac{T_0}{t}\Big)^{\lambda/2} ||\r(T_0)|| =  O(t^{-\lambda/2}).
    \end{align*}
Hence, $||\r(t)||^2/t = O(t^{-\lambda -1})$. Now using the above argument we used for $t^{-\theta}$, we obtain that $||\r(t)|| = O(t^{-\min(\lambda, \lambda)}) = O(t^{-\lambda})$.

Combining all the cases we get
$||\r(t)|| = O(t^{-\min(\lambda, 1, \theta-1)})$. 
\end{proof}

\subsection{Bias Characterisation}
\begin{theorem}\label{app:thm:main_1}
Let \Cref{ass:sepmargin,ass:support} hold and suppose $|b(t)|\le K$ for all $t\ge1$. By \Cref{thm:rho_t_convergence}, there are constants $C>0$ and $T_0$ such that $\|\brho(t)-\tilde{\w}\|\le kt^{-p}$ for $t\ge T_0$, where $p=\min(\lambda,1,\theta-1)$. For $t_0\ge T_0$, define the following quantities.
\begin{itemize}
    \item $\tilde{A}_S^+=\sum_{i\in\S^+}\exp(-\tilde{\w}^\top\x_i)$ and $\tilde{A}_S^-=\sum_{i\in\S^-}\exp(\tilde{\w}^\top\x_i)$.
    \item $\delta=\sqrt{\tilde{A}_S^-/\tilde{A}_S^+}$.
    \item $\theta^+=\min_{i\in I^+\setminus\S^+}\hat{\w}^\top\x_i>1$ and $\theta^-=\max_{i\in I^-\setminus\S^-}\hat{\w}^\top\x_i<-1$.
    \item $C_2=|I^+\setminus\S^+|\exp(RB)$ and $C_3=|I^-\setminus\S^-|\exp(RB)$.
    \item $\bar A_{\S}^+=\exp(kRt_0^{-p}+C_2t_0^{1-\theta^+}/\tilde{A}_S^+)\tilde{A}_S^+$ and $\bar A_{\S}^-=\exp(-kRt_0^{-p})\tilde{A}_S^-$.
    \item $\underline{A}_{\S}^+=\exp(-kRt_0^{-p})\tilde{A}_S^+$ and $\underline{A}_{\S}^-=\exp(kRt_0^{-p}+C_3t_0^{1+\theta^-}/\tilde{A}_S^-)\tilde{A}_S^-$.
    \item $\bar\delta=\sqrt{\bar A_{\S}^-/\bar A_{\S}^+}$ and $\underline\delta=\sqrt{\underline{A}_{\S}^-/\underline{A}_{\S}^+}$.
    \item $M=\max(K,|\log\bar\delta|,|\log\underline\delta|)$.
\end{itemize}
Suppose $t_0\ge\eta(\bar A_{\S}^++\bar A_{\S}^-+\underline{A}_{\S}^++\underline{A}_{\S}^-)\exp(2M)$, and let $\bar b(t)$ and $\underline b(t)$ be defined by $\bar b(t_0)=\underline b(t_0)=b(t_0)$ and
\begin{align}\label{app:eq:biasdyn}
\bar b(t+1)&=\bar b(t)+\frac{\eta}{t}\left(\bar A_{\S}^+\exp(-\bar b(t))-\bar A_{\S}^-\exp(\bar b(t))\right),\\
\underline b(t+1)&=\underline b(t)+\frac{\eta}{t}\left(\underline{A}_{\S}^+\exp(-\underline b(t))-\underline{A}_{\S}^-\exp(\underline b(t))\right).\nonumber
\end{align}
Then for all $t\ge t_0$,
\begin{align}\label{app:eq:biasevol}
\underline b(t)\le b(t)\le\bar b(t),
\end{align}
where $\bar b(t)$ converges monotonically to $-\log\bar\delta$ with $|\bar b(t)+\log\bar\delta|\le|b(t_0)+\log\bar\delta|(t_0/t)^{2\eta\sqrt{\bar A_{\S}^+\bar A_{\S}^-}}$, and likewise $\underline b(t)$ converges to $-\log\underline\delta$. Consequently, $b(t)$ converges to $b_\infty=-\log\delta$.
\end{theorem}

\begin{proof}
From \Cref{thm:rho_t_convergence}, we have $\|\r(t)\|=O(t^{-p})$ with $p=\min\{\lambda,\theta-1,1\}$, where $\r(t)=\brho(t)-\tilde{\w}$. Hence, there exist constants $C>0$ and $T_0$ such that $\brho(t)=\tilde{\w}+\boldsymbol{\varepsilon}(t)$, where $\|\boldsymbol{\varepsilon}(t)\|\le kt^{-p}$ for all $t\ge T_0$.

Let $I^+$ and $I^-$ denote the indices of the positive and negative data points, respectively, so that $y_i=1$ for $i\in I^+$ and $y_i=-1$ for $i\in I^-$. Let $\S^+\subseteq I^+$ and $\S^-\subseteq I^-$ denote the corresponding support-vector indices. The gradient descent dynamics of the bias are
\begin{align*}
b(t+1)-b(t)
&=-\eta\frac{\partial}{\partial b}\sum_i\exp\bigl(-y_i(\w(t)^\top\x_i+b)\bigr)\Big|_{b=b(t)}\\
&=\eta\left(\sum_{i\in I^+}\exp(-\w(t)^\top\x_i-b(t))-\sum_{i\in I^-}\exp(\w(t)^\top\x_i+b(t))\right).
\end{align*}
Substituting $\w(t)=\hat{\w}\log t+\brho(t)$, we obtain $b(t+1)-b(t)=\eta\big(\phi(t,b(t))-\psi(t,b(t))\big)$, where
\begin{align*}
\phi(t,b)&=\sum_{i\in I^+}t^{-\hat{\w}^\top\x_i}\exp(-\brho(t)^\top\x_i)\exp(-b),\\
\psi(t,b)&=\sum_{i\in I^-}t^{\hat{\w}^\top\x_i}\exp(\brho(t)^\top\x_i)\exp(b).
\end{align*}

\textbf{Bounds on $\phi$ and $\psi$.} Fix $t\ge t_0$. Splitting each sum into support and non-support vectors gives
\begin{align*}
\phi(t,b)&=\sum_{i\in\S^+}t^{-\hat{\w}^\top\x_i}\exp(-\brho(t)^\top\x_i)\exp(-b)+\sum_{i\in I^+\setminus\S^+}t^{-\hat{\w}^\top\x_i}\exp(-\brho(t)^\top\x_i)\exp(-b),\\
\psi(t,b)&=\sum_{i\in\S^-}t^{\hat{\w}^\top\x_i}\exp(\brho(t)^\top\x_i)\exp(b)+\sum_{i\in I^-\setminus\S^-}t^{\hat{\w}^\top\x_i}\exp(\brho(t)^\top\x_i)\exp(b).
\end{align*}
For $i\in\S^+$, we have $\hat{\w}^\top\x_i=1$, and for $i\in\S^-$, we have $\hat{\w}^\top\x_i=-1$. For non-support vectors and $t\ge1$, we have $t^{-\hat{\w}^\top\x_i}\le t^{-\theta^+}$ for $i\in I^+\setminus\S^+$ and $t^{\hat{\w}^\top\x_i}\le t^{\theta^-}$ for $i\in I^-\setminus\S^-$. By Cauchy--Schwarz, $|\brho(t)^\top\x_i|\le\|\brho(t)\|\|\x_i\|\le RB$ and $|\boldsymbol{\varepsilon}(t)^\top\x_i|\le\|\boldsymbol{\varepsilon}(t)\|\|\x_i\|\le kRt^{-p}$.

Dropping the nonnegative non-support-vector terms gives the lower bounds
\begin{align*}
\phi(t,b)&\ge\frac1t\exp(-kRt^{-p})\tilde{A}_S^+\exp(-b)\ge\frac1t\underline{A}_{\S}^+\exp(-b),\\
\psi(t,b)&\ge\frac1t\exp(-kRt^{-p})\tilde{A}_S^-\exp(b)\ge\frac1t\bar A_{\S}^-\exp(b).
\end{align*}
Keeping them gives the upper bounds
\begin{align*}
\phi(t,b)&\le\frac1t\Big(\exp(kRt^{-p})+\frac{C_2}{\tilde{A}_S^+}t^{1-\theta^+}\Big)\tilde{A}_S^+\exp(-b)\le\frac1t\bar A_{\S}^+\exp(-b),\\
\psi(t,b)&\le\frac1t\Big(\exp(kRt^{-p})+\frac{C_3}{\tilde{A}_S^-}t^{1+\theta^-}\Big)\tilde{A}_S^-\exp(b)\le\frac1t\underline{A}_{\S}^-\exp(b).
\end{align*}
In the last inequality of each line we used $\exp(a)+y\le \exp(a+y)$ for $a,y\ge0$, and the fact that each exponent decreases in $t$, so it is largest at $t=t_0$. Combining these bounds, for all $t\ge t_0$,
\begin{align}\label{app:eq:bias-differential-bounds}
\frac{\eta}{t}\left(\underline{A}_{\S}^+\exp(-b(t))-\underline{A}_{\S}^-\exp(b(t))\right)\le b(t+1)-b(t)\le\frac{\eta}{t}\left(\bar A_{\S}^+\exp(-b(t))-\bar A_{\S}^-\exp(b(t))\right).
\end{align}

\textbf{Convergence of upper and lower envelopes.} Since $\bar A_{\S}^\pm=\sqrt{\bar A_{\S}^+\bar A_{\S}^-}\exp(\mp\log\bar\delta)$, the upper recursion in \Cref{app:eq:biasdyn} can be written as
\begin{align*}
\bar b(t+1)+\log\bar\delta&=\big(\bar b(t)+\log\bar\delta\big)-\frac{2\eta\sqrt{\bar A_{\S}^+\bar A_{\S}^-}}{t}\sinh\big(\bar b(t)+\log\bar\delta\big).
\end{align*}
That is, $\bar b(t+1)+\log\bar\delta$ is the image of $\bar b(t)+\log\bar\delta$ under the map $x\mapsto x-\frac{2\eta}{t}\sqrt{\bar A_{\S}^+\bar A_{\S}^-}\sinh x$. This map is odd and vanishes at $0$. For $|x|\le2M$, its derivative satisfies
\begin{align*}
1-\frac{2\eta}{t}\sqrt{\bar A_{\S}^+\bar A_{\S}^-}\cosh x&\ge1-\frac{\eta}{t}\big(\bar A_{\S}^++\bar A_{\S}^-\big)\cosh x\\
&\ge1-\frac{\eta}{t_0}\big(\bar A_{\S}^++\bar A_{\S}^-\big)\exp(2M)\\
&\ge0,
\end{align*}
where the first inequality uses $2\sqrt{\bar A_{\S}^+\bar A_{\S}^-}\le\bar A_{\S}^++\bar A_{\S}^-$, the second uses $t\ge t_0$ and $\cosh x\le\exp(2M)$, and the third uses the condition on $t_0$. Hence the map is nondecreasing on $[-2M,2M]$.

Now suppose $0\le\bar b(t)+\log\bar\delta\le2M$. Since the map is nondecreasing and vanishes at $0$,
\begin{align*}
\bar b(t+1)+\log\bar\delta\ge0.
\end{align*}
Since $\sinh x\ge x$ for $x\ge0$,
\begin{align*}
\bar b(t+1)+\log\bar\delta&\leq \big(\bar b(t)+\log\bar\delta\big)-\frac{2\eta\sqrt{\bar A_{\S}^+\bar A_{\S}^-}}{t}\big(\bar b(t)+\log\bar\delta\big)\\
&= \left(1-\frac{2\eta\sqrt{\bar A_{\S}^+\bar A_{\S}^-}}{t}\right)\big(\bar b(t)+\log\bar\delta\big)\\
&\le\bar b(t)+\log\bar\delta\le2M.
\end{align*}
Since the map is odd, the same holds with all signs reversed when $-2M\le\bar b(t)+\log\bar\delta\le0$. Moreover,
\begin{align*}
|\bar b(t_0)+\log\bar\delta|\le|b(t_0)|+|\log\bar\delta|\le2M.
\end{align*}
By induction on $t$, $\bar b(t)+\log\bar\delta$ never changes sign, and for all $t\ge t_0$,
\begin{align*}
|\bar b(t)+\log\bar\delta|&\le|b(t_0)+\log\bar\delta|\prod_{s=t_0}^{t-1}\Big(1-\frac{2\eta\sqrt{\bar A_{\S}^+\bar A_{\S}^-}}{s}\Big)\\
&\le|b(t_0)+\log\bar\delta|\exp\Big(-2\eta\sqrt{\bar A_{\S}^+\bar A_{\S}^-}\sum_{s=t_0}^{t-1}\frac1s\Big)\\
&\le|b(t_0)+\log\bar\delta|\Big(\frac{t_0}{t}\Big)^{2\eta\sqrt{\bar A_{\S}^+\bar A_{\S}^-}},
\end{align*}
where the second inequality uses $1-x\le\exp(-x)$ and the third uses $\sum_{s=t_0}^{t-1}\frac1s\ge\log\frac{t}{t_0}$. This is the stated bound. Since $\bar b(t)+\log\bar\delta$ keeps its sign and its absolute value is nonincreasing, $\bar b(t)$ moves monotonically from $b(t_0)$ toward $-\log\bar\delta$. In particular,
\begin{align*}
|\bar b(t)|\le\max\big(|b(t_0)|,|\log\bar\delta|\big)\le M.
\end{align*}
The same argument applies to $\underline b(t)$.

\textbf{Upper and lower envelopes bound $b(t)$.} Let $F(t,b)=b+\frac{\eta}{t}\big(\bar A_{\S}^+\exp(-b)-\bar A_{\S}^-\exp(b)\big)$. For $b\in[-M,M]$,
\begin{align*}
\frac{\partial F}{\partial b}&=1-\frac{\eta}{t}\big(\bar A_{\S}^+\exp(-b)+\bar A_{\S}^-\exp(b)\big)\\
&\ge1-\frac{\eta}{t_0}\big(\bar A_{\S}^++\bar A_{\S}^-\big)\exp(M)\\
&\ge0,
\end{align*}
so $F(t,\cdot)$ is nondecreasing on $[-M,M]$. By \Cref{app:eq:bias-differential-bounds}, $b(t+1)\le F(t,b(t))$. Moreover, $|b(t)|\le K\le M$, and $|\bar b(t)|\le M$ by the previous step. At $t_0$, $b(t_0)=\bar b(t_0)$. If $b(t)\le\bar b(t)$, then
\begin{align*}
b(t+1)\le F(t,b(t)) \le F(t,\bar b(t))=\bar b(t+1).
\end{align*}
By induction, $b(t)\le\bar b(t)$ for all $t\ge t_0$. The same argument with the lower bound in \Cref{app:eq:bias-differential-bounds} gives $b(t)\ge\underline b(t)$. This proves \Cref{app:eq:biasevol}.

\textbf{Limits of the bounds sandwich the limit of $b(t)$.} By \Cref{app:eq:biasevol} and the convergence of the envelopes,
\begin{align*}
-\log\underline\delta\le\liminf_{t\to\infty}b(t)\le\limsup_{t\to\infty}b(t)\le-\log\bar\delta .
\end{align*}
By the definitions of the coefficients,
\begin{align*}
-\log\bar\delta&=-\log\delta+kRt_0^{-p}+\frac{C_2}{2\tilde{A}_S^+}t_0^{1-\theta^+},\\
-\log\underline\delta&=-\log\delta-kRt_0^{-p}-\frac{C_3}{2\tilde{A}_S^-}t_0^{1+\theta^-}.
\end{align*}
These inequalities hold for every admissible $t_0$. All sufficiently large $t_0$ are admissible, because the coefficients and $M$ decrease to finite limits as $t_0\to\infty$. Since $p>0$, $\theta^+>1$ and $\theta^-<-1$, both bounds tend to $-\log\delta$ as $t_0\to\infty$. Hence $\lim_{t\to\infty}b(t)=-\log\delta$.
\end{proof}

\subsection{Delayed Generalization Time}
\begin{theorem}\label{app:thm:main_3}
Let \Cref{ass:sepmargin} and \Cref{ass:support} hold, and suppose $|b(t)|\le K$ for all $t\ge1$. Let the test distribution be $\mathcal{P}_{conc}$: the classes are balanced, test points satisfy $||\x||\le R$, and within each class the margin $s=y\hat{\w}^\top\x/||\hat{\w}||$ is uniform on $[\gamma/2,\alpha\gamma]$. Fix $\epsilon\in(0,\frac12)$ and $\vartheta\in(0,1)$, and suppose that $|b(t)|\ge(1-\vartheta)|\log\delta|$ for all $t\ge T_2=T_2(\vartheta)$. If
\begin{align*}
\log T_2<\frac{(1-\vartheta)|\log\delta|-RB}{1+2\epsilon(2\alpha-1)},
\end{align*}
then
\begin{align*}
T_Q(\epsilon)\ \ge\ \exp\left(\frac{(1-\vartheta)|\log\delta|-RB}{1+2\epsilon(2\alpha-1)}\right).
\end{align*}
\end{theorem}

\begin{proof}
Let $L:=\frac{(1-\vartheta)|\log\delta|-RB}{1+2\epsilon(2\alpha-1)}$. By \Cref{eq:asymptotic_decomposition}, $\w(t)=\hat{\w}\log t+\brho(t)$ for $t\ge1$, and $||\hat{\w}||=2/\gamma$ by \Cref{ass:sepmargin}. For a test point $(\x,y)$ let $s=y\hat{\w}^\top\x/||\hat{\w}||\in[\gamma/2,\alpha\gamma]$ be its signed margin, so that $y\hat{\w}^\top\x=2s/\gamma$. Using $|y\brho(t)^\top\x|\le RB$, the prediction satisfies
\begin{equation}\label{app:eq:margin_split}
\frac{2s}{\gamma}\log t-RB+y\,b(t)\ \le\ y\big(\w(t)^\top\x+b(t)\big)\ \le\ \frac{2s}{\gamma}\log t+RB+y\,b(t).
\end{equation}

\textbf{The test accuracy eventually reaches $1$.} Since $s\ge\gamma/2$ and $|b(t)|\le K$, the left inequality in \Cref{app:eq:margin_split} gives $y(\w(t)^\top\x+b(t))\ge\log t-RB-K$. This is positive for every test point once $\log t>RB+K$. Hence $Q(t)=1$ for all large $t$, so $\sup_uQ(u)=1$ and $T_Q(\epsilon)$ is the first time after which $Q$ stays at or above $1-\epsilon$.

\textbf{The test accuracy is below $1-\epsilon$ before $\exp(L)$.} Fix an integer $t$ with $T_2\le t<\exp(L)$. The bias satisfies $|b(t)|\ge(1-\vartheta)|\log\delta|>RB$, where the last inequality holds because $L>\log T_2\ge0$. In particular $b(t)\neq0$, and the decision boundary is shifted off the origin toward the class $y=-\operatorname{sign}b(t)$.

For a test point of that class, $y\,b(t)=-|b(t)|$, and the right inequality in \Cref{app:eq:margin_split} gives
\begin{align*}
y\big(\w(t)^\top\x+b(t)\big)\ \le\ \frac{2s}{\gamma}\log t+RB-|b(t)| .
\end{align*}
If $t=1$, the right-hand side is negative for every $s$, so the whole class is misclassified. If $t\ge2$, the right-hand side is negative whenever $s<v:=\frac{\gamma\,(|b(t)|-RB)}{2\log t}$. In other words, every point of this class whose margin is smaller than $v$ lies on the wrong side of the shifted hyperplane.

We now count these points under the population. Half of the test population belongs to this class, and within it the margin $s$ is uniform on $[\gamma/2,\alpha\gamma]$. So the fraction of the class that is misclassified is at least $\min\{1,(v-\gamma/2)/(\alpha\gamma-\gamma/2)\}$. Since $\log t<L\le(|b(t)|-RB)/(1+2\epsilon(2\alpha-1))$,
\begin{align*}
v\ >\ \frac{\gamma}{2}\big(1+2\epsilon(2\alpha-1)\big),
\qquad\text{hence}\qquad
\frac{v-\gamma/2}{\alpha\gamma-\gamma/2}\ >\ \frac{\gamma\,\epsilon\,(2\alpha-1)}{\frac{\gamma}{2}(2\alpha-1)}=2\epsilon .
\end{align*}
Because $\epsilon<\frac12$, $2\epsilon<1$, so more than a fraction $2\epsilon$ of this class is misclassified. The population error is therefore more than $\frac12\cdot2\epsilon=\epsilon$, that is, $Q(t)<1-\epsilon=(1-\epsilon)\sup_uQ(u)$.

\textbf{Combining the above.} We have shown that $Q(t)<1-\epsilon$ at every integer $t\in[T_2,\exp(L))$. This set is nonempty because it contains $T_2$. From the first step and the definition of $T_Q(\epsilon)$, the generalization time must exceed the largest integer in this set, which is at least $\exp(L)-1$. Hence $T_Q(\epsilon)\ge \exp(L)$.
\end{proof}

\section{Experimental Setup}
\label{app:experimental_setup}

This section specifies every quantity needed to reproduce the experiments of \Cref{sec:sec5}: the data-generating process, the planting of support vectors, the model, the loss, the optimizer and its step size, the numerical safeguards, the evaluation protocol, and the computing environment. All hyperparameters are collected in \Cref{app:tab:hyper}. We also verify in \Cref{app:sec:stepsize} that the step size used in every synthetic run guarantees monotone decrease of the exponential loss, which is the property the step-size assumption of \Cref{sec:model} is meant to ensure. All synthetic experiments are carried out in $\R^2$, where the training data, its support vectors and the test data are fully controlled.

\subsection{Synthetic Data Generation}
\label{app:sec:synthdata}

The ground-truth separating hyperplane passes through the origin and is defined by the unit-norm vector
\begin{align*}
\w_{\text{true}} = \frac{1}{\sqrt{2}}[1, 1]^\top .
\end{align*}
For a point $\x$, the quantity $\w_{\text{true}}^\top\x$ is its signed distance to this hyperplane. All candidate points are drawn independently and uniformly from the square $[-5,5]^2$. We use the margin parameter $\gamma=10^{-3}$ of \Cref{ass:sepmargin} and the sensitivity parameter $\alpha=10$ of $\mathcal{P}_{conc}$ in every run.

\paragraph{Random number generation.}
All randomness comes from the global \texttt{NumPy} generator, which is reset with \texttt{np.random.seed(42)} at the start of every run. The sets are then generated in a fixed order: the positive points drawn from $\mathcal{P}$, the negative points drawn from $\mathcal{P}$, the planted positive support vectors, the planted negative support vectors, the positive test points and the negative test points. Since the generator is reset before each run, two runs that share the same class sizes and support-vector counts are trained on identical data.

\paragraph{Training data and planted support vectors.}
The training set $\mathcal{D}$ consists of two types of points.

\emph{Points drawn from $\mathcal{P}$.} A positive point is obtained by drawing candidates uniformly from the square and accepting the first candidate that satisfies $\w_{\text{true}}^\top\x \geq \gamma/2$. A negative point is obtained in the same way with the condition $\w_{\text{true}}^\top\x \leq -\gamma/2$. Since the candidates are drawn from a continuous distribution, with probability one every such point satisfies the inequality strictly, i.e.\ $y\,\w_{\text{true}}^\top\x > \gamma/2$.

\emph{Planted support vectors.} A support vector with label $y$ is planted by drawing a point $\x$ uniformly from the square and projecting it orthogonally onto the hyperplane $\w_{\text{true}}^\top\x = y\gamma/2$:
\begin{align*}
\x \;\longmapsto\; \x - \left(\w_{\text{true}}^\top\x - \frac{y\gamma}{2}\right)\w_{\text{true}} .
\end{align*}
Since $\w_{\text{true}}$ has unit norm, the projected point satisfies $\w_{\text{true}}^\top\x = y\gamma/2$, while its component orthogonal to $\w_{\text{true}}$ is unchanged. In floating point, the largest deviation $\max|\w_{\text{true}}^\top\x - y\gamma/2|$ over all planted points is below $2.4\times10^{-15}$ in every run. Hence every planted point lies at distance $\gamma/2$ from the ground-truth hyperplane, on the side of its class. The orthogonal components of the planted points are spread over the whole admissible range: in every run they cover more than $[-6.4,6.5]$ out of $[-5\sqrt2,5\sqrt2]\approx[-7.07,7.07]$ in each class.

The positive class $I^+$ consists of $|\S^+|$ planted points and $N_+ - |\S^+|$ points drawn from $\mathcal{P}$, and the negative class $I^-$ consists of $|\S^-|$ planted points and $N_- - |\S^-|$ points drawn from $\mathcal{P}$. The training set size is $N = N_+ + N_- = 2000$ in all synthetic runs.

\paragraph{Support vectors of the training set.}
By construction, every training point satisfies $y_i\w_{\text{true}}^\top\x_i \geq \gamma/2$, with equality exactly for the planted points. Moreover, the planted points of both classes lie on the two margin hyperplanes and extend along them in both directions orthogonal to $\w_{\text{true}}$. Consequently, any separator that is tilted with respect to $\w_{\text{true}}$, or shifted away from the origin, brings some planted point closer than $\gamma/2$ to it. The maximum-margin separator of $\mathcal{D}$ is therefore $\w_{\text{true}}$ with zero bias, so \Cref{ass:sepmargin} holds with $\w_* = \w_{\text{true}}$ and $\hat{\w} = \frac{2}{\gamma}\w_{\text{true}}$. Since $y_i\hat{\w}^\top\x_i = 1$ holds exactly for the planted points and $y_i\hat{\w}^\top\x_i > 1$ for all other points, the support vectors of $\mathcal{D}$ are exactly the planted points: the planted positive points form $\S^+$ and the planted negative points form $\S^-$. This allows us to set $|\S^+|$ and $|\S^-|$ directly, independently of $N_+$ and $N_-$.

For every planted configuration, the smallest normalized margin of a non-support vector is $\theta=\min_{i\notin\S}y_i\hat{\w}^\top\x_i=1.2581$, so $\theta-1\approx0.258$. By \Cref{thm:rho_t_convergence}, the residual therefore converges at least at the rate $t^{-p}$ with $p=\min(\lambda,1,\theta-1)\le0.258$ in these runs.

\paragraph{Training data without planted support vectors.}
In the experiment of \Cref{fig:delayed convergence}, no support vectors are planted. Its support vectors are then the points of $\mathcal{D}$ closest to the maximum-margin separator, which is not constrained to pass through the origin. For the seed used, a linear SVM fitted with \texttt{scikit-learn} (\texttt{SVC}, linear kernel, $C=10^8$) has geometric margin $1.73\times10^{-3}$ and offset $|b|/\|\w\|\approx8.3\times10^{-4}$ from the origin. This experiment therefore illustrates the dependence of the delay on the test distribution. The quantitative predictions of \Cref{sec4}, which rely on \Cref{ass:sepmargin}, are tested on the planted configurations.

\paragraph{Test data.}
Test points are generated by the same rejection procedure, with the acceptance conditions below. Both test sets contain $1000$ positive and $1000$ negative points.

\emph{Standard test set, drawn from $\mathcal{P}$.} Positive test points satisfy $\w_{\text{true}}^\top\x > \gamma/2$ and negative test points satisfy $\w_{\text{true}}^\top\x < -\gamma/2$. This is the distribution of the points of $\mathcal{D}$ that are not planted.

\emph{Concentrated test set, drawn from $\mathcal{P}_{conc}$.} Positive test points satisfy
\begin{align*}
\frac{\gamma}{2} < \w_{\text{true}}^\top\x \leq \alpha \gamma,
\end{align*}
and negative test points satisfy
\begin{align*}
-\alpha \gamma \leq \w_{\text{true}}^\top\x < -\frac{\gamma}{2}.
\end{align*}
With $\alpha=10$ and $\gamma=10^{-3}$, this is a band of width $9.5\times10^{-3}$ adjacent to each margin hyperplane. $\mathcal{P}_{conc}$ is thus $\mathcal{P}$ restricted to a thin band next to the margin hyperplanes. Its points lie closer to the decision boundary than the training data, while remaining correctly classified by the ground-truth hyperplane and by the maximum-margin separator of $\mathcal{D}$. No test point lies on a margin hyperplane, so no test point coincides with a planted support vector.

\begin{figure}[h]
    \centering
    \includegraphics[width=0.5\linewidth]{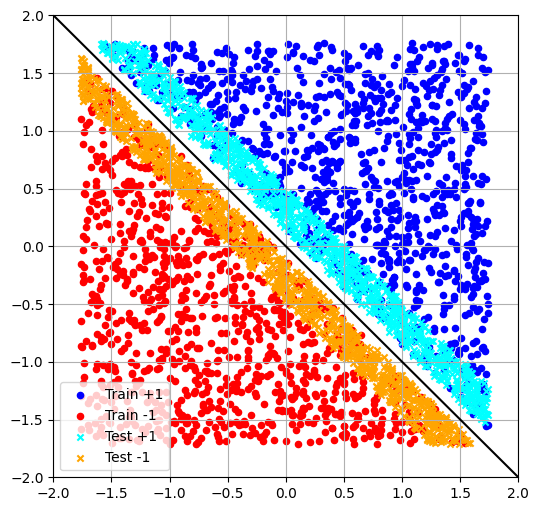}
    \caption{Train and test data distribution}
    \label{app:fig:dataset}
\end{figure}

\subsection{Experiments}
\label{app:sec:runs}

The experiments differ only in the class sizes $N_+$ and $N_-$ (the population) and in the numbers of planted support vectors $|\S^+|$ and $|\S^-|$. The training set size, the test set size and all optimization hyperparameters are the same in all runs (\Cref{app:tab:hyper}). We write $D(N_+,N_-)$ for the population and $S(|\S^+|,|\S^-|)$ for the planted support vectors.

\begin{itemize}
    \item \textbf{Standard versus concentrated test (\Cref{fig:delayed convergence}).} $D(1000,1000)$ with no planted support vectors. The training data and the concentrated test data of this experiment are shown in \Cref{app:fig:dataset}. The trained model is evaluated on the standard test set (\Cref{fig:accuracy_same_as_train_dim_2}) and on the concentrated test set (\Cref{fig:accuracy_dist_shift_dim_2}).
    \item \textbf{Balanced population, varying support vectors (\Cref{fig:4} and the top row of \Cref{fig:symm_asymm+_asymm-}).} $D(1000,1000)$ with $S(350,350)$, $S(367,333)$, $S(382,318)$ and $S(400,300)$. The total number of planted points is $700$ in each run, so only the split between the classes changes. The predicted bias limits $\frac12\log(|\S^+|/|\S^-|)$ are $0$, $0.049$, $0.092$ and $0.144$, respectively. Since $N_+ = N_-$, the bias is not pulled away from zero by the population during Phase~1, so any later deviation of $b(t)$ is due to the support vectors alone.
    \item \textbf{Imbalanced population (middle and bottom rows of \Cref{fig:symm_asymm+_asymm-}).} $D(1500,500)$ in the middle row and $D(500,1500)$ in the bottom row, each with $|\S^+|\neq|\S^-|$, for example $S(300,400)$. Across runs, the support-vector imbalance either agrees or disagrees in sign with the population imbalance. Runs with the same $|\S^+|$ and $|\S^-|$ but opposite population imbalance differ only in the numbers of points drawn from $\mathcal{P}$ in each class, $N_+ - |\S^+|$ and $N_- - |\S^-|$. The configuration of each curve is given in the caption and legend of the figure.
\end{itemize}

In all experiments with planted support vectors, the trained model is evaluated on the concentrated test set.

\subsection{Model and Loss}
\label{app:sec:model}

We use the linear model with bias of \Cref{sec:model}, $\hat{y}(\x;\w,b)=\w^\top \x+b$, where $\w\in\R^2$ and $b\in\R$. We write $\boldsymbol{\vartheta}=(\w,b)\in\R^3$ for the parameters and $\mathbf{z}_i=y_i(\x_i^\top,1)^\top$ for the augmented, label-absorbed data, so that $y_i\hat{y}(\x_i;\w,b)=\mathbf{z}_i^\top\boldsymbol{\vartheta}$. We minimize the exponential loss $\L$ of \Cref{sec:model} normalized by the training set size,
\begin{align*}
    \frac{1}{N}\L(\boldsymbol{\vartheta}) = \frac{1}{N}\sum_{i=1}^N \exp\big(-\mathbf{z}_i^\top\boldsymbol{\vartheta}\big).
\end{align*}
The gradients are computed in closed form, without automatic differentiation:
\begin{align*}
    \nabla_{\w}\frac{1}{N}\L(\w,b) &= -\frac{1}{N}\sum_{i=1}^N y_i \exp\big(-y_i(\w^\top\x_i+b)\big)\,\x_i,\\
    \nabla_{b}\frac{1}{N}\L(\w,b) &= -\frac{1}{N}\sum_{i=1}^N y_i \exp\big(-y_i(\w^\top\x_i+b)\big).
\end{align*}

\subsection{Optimization Procedure}
\label{app:sec:optim}

We optimize the loss with full-batch gradient descent, initialized at $\w(0) = \mathbf{0}$ and $b(0) = 0$:
\begin{align*}
    \w(t+1) &= \w(t) - \eta_{\mathrm{lr}} \nabla_{\w} \frac{1}{N}\L(\w(t), b(t)),\\
    b(t+1) &= b(t) - \eta_{\mathrm{lr}} \nabla_{b} \frac{1}{N}\L(\w(t), b(t)),
\end{align*}
with learning rate $\eta_{\mathrm{lr}}=10^{-2}$ for $T=10^7$ iterations. The weights and the bias are updated simultaneously, using gradients evaluated at $(\w(t), b(t))$. Every iteration uses all $N$ training points, so an iteration coincides with an epoch. No explicit regularization, momentum, weight decay, early stopping or learning-rate scheduling is used.

Gradient descent on $\frac1N\L$ with learning rate $\eta_{\mathrm{lr}}$ produces exactly the same iterates as gradient descent on $\L$ with step size
\begin{align}\label{app:eq:eta_eff}
    \eta=\frac{\eta_{\mathrm{lr}}}{N}=\frac{10^{-2}}{2000}=5\times10^{-6}.
\end{align}
All statements of \Cref{sec4} refer to this step size $\eta$ on the unnormalized loss $\L$.

\subsection{Step Size and Monotone Descent}
\label{app:sec:stepsize}

The step-size assumption of \Cref{sec:model} is used to guarantee that the exponential loss decreases monotonically along the gradient descent trajectory.

\subsection{Evaluation Metrics and Logging}
\label{app:sec:logging}

A point is predicted as positive if $\w(t)^\top\x + b(t) \geq 0$ and as negative otherwise, and the accuracy is the fraction of correctly classified points. This agrees with the criterion $y\hat{y}(\x;\w(t),b(t))>0$ of \Cref{sec:model} except on the event $\w(t)^\top\x+b(t)=0$. Since $\w(t)\neq\mathbf{0}$ for $t\ge1$ and the test points are drawn from a continuous distribution independently of the training trajectory, this event has probability zero.

Metrics are computed on the parameters $(\w(t), b(t))$ obtained after the $t$-th update, and are recorded for $t\in\{1,101,201,\ldots,9\,999\,901\}$, i.e.\ every $100$ iterations, giving $10^5$ records per run. At each recorded iteration, we store the normalized training and test losses (computed with the clipping above), the training accuracy $P(t)$ and the test accuracy $Q(t)$. In the experiments with planted support vectors, we additionally record the following quantities describing the learned hyperplane $\{\x : \w(t)^\top\x + b(t) = 0\}$: the bias $b(t)$; the norm $\|\w(t)\|$; the cosine of the angle between $\w(t)$ and $\w_{\text{true}}$, $\w(t)^\top\w_{\text{true}}/\|\w(t)\|$; and the distance of the hyperplane from the origin, $D(t) = |b(t)|/\|\w(t)\|$.

\subsection{MNIST Experiment}
\label{app:sec:mnist}

We verify the predicted bias behavior in a deep network on a binary classification task between the visually similar MNIST digits $4$ and $9$. We perform two runs that differ only in the assignment of labels: in \emph{4 vs 9}, digit $4$ is the positive class and digit $9$ the negative class; in \emph{9 vs 4}, digit $9$ is the positive class and digit $4$ the negative class.

\paragraph{Data.}
We take all images of digits $4$ and $9$ from the MNIST training split as the training set ($5842$ and $5949$ images) and all images of these digits from the MNIST test split as the test set ($982$ and $1009$ images). Each $28\times28$ image is flattened into a vector in $\R^{784}$ and its pixel intensities are divided by $255$, so that $\x\in[0,1]^{784}$ and $\|\x\|\le28$. No other preprocessing (centering, standardization or augmentation) is applied. The test set is used as is and plays the role of the standard test set; no concentrated test set is constructed.

\paragraph{Enforcing linear separability.}
Our analysis assumes a linearly separable training set, which the raw training images of digits $4$ and $9$ are not. We therefore remove offending points with the following iterative procedure. A soft-margin linear SVM with hinge loss (\texttt{scikit-learn} \texttt{LinearSVC}, $C=10$, dual formulation, tolerance $10^{-6}$, at most $2\times10^5$ iterations) is fit to the current training set. Every training point $(\x_i, y_i)$ with a non-positive signed margin under the fitted classifier is removed. The SVM is then refit on the remaining points, and the procedure is repeated, for at most $100$ rounds, until the fitted SVM classifies every remaining training point correctly. The remaining set is then linearly separable by construction, since a linear classifier separates it. All subsequent steps use only this set, and the test set is not modified. The procedure removes $99$ images of digit $4$ and $87$ of digit $9$ in the \emph{4 vs 9} run, and $100$ and $89$ in the \emph{9 vs 4} run.

\paragraph{Support vectors and population.}
On the separable training set, we fit a linear SVM with bias and a large penalty parameter (\texttt{scikit-learn} \texttt{SVC}, linear kernel, $C=10^5$, tolerance $10^{-4}$), which approximates the hard-margin SVM. Its support vectors define $\S^+$ and $\S^-$. The resulting class sizes and support-vector counts are: \emph{4 vs 9}: $N_+=5743$, $|\S^+|=195$, $N_-=5862$, $|\S^-|=178$; \emph{9 vs 4}: $N_+=5860$, $|\S^+|=173$, $N_-=5742$, $|\S^-|=189$. This SVM is used only to identify the support vectors; it plays no role in training the network. These support vectors are computed in input space, whereas the bias dynamics of \Cref{lem:popdyn} act on the hidden representation $\h(\x)$; the input-space counts are used only to indicate the direction of the support-vector imbalance. The separability procedure and the SVM fit are carried out separately in each run, so the populations and support vectors of the two runs need not coincide exactly after exchanging the classes. In both runs, the population imbalance and the support-vector imbalance point in opposite directions.

\paragraph{Model.}
The network has one fully connected hidden layer of width $64$ with ReLU activation, followed by a single linear output neuron with no output activation:
\begin{align*}
    f(\x) = \W^\top \h(\x) + \b,\qquad \h(\x)=\max\big(\mathbf{0},\,\W_1\x+\b_1\big),
\end{align*}
where $\W_1\in\R^{64\times784}$ and $\b_1\in\R^{64}$ are the hidden-layer parameters, and $\W\in\R^{64}$ and $\b\in\R$ are the weight vector and bias of the output neuron. The network has $50\,305$ parameters. All weights and the hidden-layer bias use the default initialization of \texttt{PyTorch} for linear layers (uniform on $[-1/\sqrt{\mathrm{fan\_in}},1/\sqrt{\mathrm{fan\_in}}]$) with seed $0$ (\texttt{torch.manual\_seed(0)}), and the output bias is initialized at $\b(0) = 0$.

\paragraph{Loss and optimization.}
We minimize the exponential loss normalized by the training set size,
\begin{align*}
    \frac{1}{N}\sum_{i=1}^N \exp\big(-y_i f(\x_i)\big),
\end{align*}
with full-batch gradient descent (\texttt{torch.optim.SGD} on the whole training set, without momentum, dampening, Nesterov acceleration or weight decay), learning rate $\eta_{\mathrm{lr}}=10^{-2}$, for $2\times10^4$ steps. This corresponds to step size $\eta=\eta_{\mathrm{lr}}/N\approx8.6\times10^{-7}$ on the unnormalized loss, for both runs. All parameters of the network, including those of the hidden layer, are trained jointly. No margin clipping is applied. All computations are carried out in double precision (\texttt{float64}). Since the network is not linear in its parameters, \Cref{app:sec:stepsize} does not apply to it, and we do not claim that the step size satisfies the assumption of \Cref{sec:model} for this model.

At every step, the update of the output bias $\b$ has the same form as \Cref{eq:biasdiffeq}, up to the factor $1/N$, with $\x_i$ replaced by the current hidden-layer representation $\h(\x_i)$ and $\w(t)$ by $\W(t)$.

\paragraph{Logging.}
Metrics are logged every $100$ steps, starting from initialization ($t=0$) and including the final step, giving $201$ records per run. At each logged step we record the training accuracy $P(t)$ and the test accuracy $Q(t)$, where a point is correctly classified if $y f(\x) > 0$; the output bias $\b(t)$; and $\b(t)/\|\W(t)\|$, whose absolute value is the distance $D(t) = |\b(t)|/\|\W(t)\|$ of the output neuron's hyperplane from the origin in the space of hidden-layer representations.

\subsection{Hyperparameters}
\label{app:sec:hyper}

\Cref{app:tab:hyper} lists all hyperparameters of the synthetic and MNIST experiments. The synthetic hyperparameters are identical across all runs; only the populations and support-vector counts listed in \Cref{app:sec:runs} change.

\begin{table}[h]
\centering
\small
\caption{Hyperparameters of all experiments. The step size $\eta$ is the step size on the unnormalized loss $\L$ used in the theory, obtained from the learning rate $\eta_{\mathrm{lr}}$ on the normalized loss $\frac1N\L$ by $\eta=\eta_{\mathrm{lr}}/N$ (\Cref{app:eq:eta_eff}).}
\label{app:tab:hyper}
\begin{tabular}{p{0.24\linewidth}p{0.32\linewidth}p{0.32\linewidth}}
\toprule
\textbf{Hyperparameter} & \textbf{Synthetic (\Cref{fig:delayed convergence,fig:4,fig:symm_asymm+_asymm-})} & \textbf{MNIST (\Cref{fig:Mnist Experiment})} \\
\midrule
\multicolumn{3}{l}{\emph{Data}} \\
Input dimension $d$ & $2$ & $784$ \\
Input domain & $[-5,5]^2$ & $[0,1]^{784}$ (pixels $/255$) \\
Ground-truth direction & $\w_{\text{true}}=\frac{1}{\sqrt2}[1,1]^\top$ & -- \\
Margin parameter $\gamma$ & $10^{-3}$ & -- \\
Sensitivity parameter $\alpha$ & $10$ & -- \\
Training set size $N$ & $2000$ & $11\,605$ (4 vs 9), $11\,602$ (9 vs 4) \\
$(N_+,N_-)$, $(|\S^+|,|\S^-|)$ & see \Cref{app:sec:runs} & see \Cref{app:sec:mnist} \\
Test set size & $1000$ per class & $982$ (digit 4), $1009$ (digit 9) \\
Test distribution & $\mathcal{P}$ or $\mathcal{P}_{conc}$ & MNIST test split ($\approx\mathcal{P}$) \\
Data seed & $42$ (\texttt{np.random.seed}) & -- (fixed MNIST splits) \\
\midrule
\multicolumn{3}{l}{\emph{Model}} \\
Architecture & $\w^\top\x+b$ & $784\to64\to1$, ReLU, linear output \\
Number of parameters & $3$ & $50\,305$ \\
Initialization & $\w(0)=\mathbf{0}$, $b(0)=0$ & \texttt{PyTorch} default; $\b(0)=0$ \\
Initialization seed & -- (deterministic) & $0$ (\texttt{torch.manual\_seed}) \\
\midrule
\multicolumn{3}{l}{\emph{Optimization}} \\
Loss & $\frac1N\sum_i\exp(-y_i(\w^\top\x_i+b))$ & $\frac1N\sum_i\exp(-y_if(\x_i))$ \\
Optimizer & full-batch GD (closed-form gradients) & full-batch GD (\texttt{torch.optim.SGD}) \\
Learning rate $\eta_{\mathrm{lr}}$ & $10^{-2}$ & $10^{-2}$ \\
Step size on $\L$, $\eta=\eta_{\mathrm{lr}}/N$ & $5\times10^{-6}$ & $\approx8.6\times10^{-7}$ \\
Batch size & $N$ (full batch) & $N$ (full batch) \\
Iterations $T$ & $10^7$ & $2\times10^4$ \\
Momentum / weight decay & $0$ / $0$ & $0$ / $0$ \\
LR schedule / early stopping & none / none & none / none \\
Margin clipping & $[-50,50]$  & none \\
Precision & \texttt{float64} & \texttt{float64} \\
\midrule
\multicolumn{3}{l}{\emph{Support vectors}} \\
Identification & planted (exact) & \texttt{SVC}, linear, $C=10^5$, tol $10^{-4}$ \\
Separability filter & -- & \texttt{LinearSVC}, hinge, $C=10$, tol $10^{-6}$, $\le2\times10^5$ iter., $\le100$ rounds \\
\midrule
\multicolumn{3}{l}{\emph{Evaluation}} \\
Prediction rule & $+1$ iff $\w^\top\x+b\ge0$ & correct iff $yf(\x)>0$ \\
Logging interval & every $100$ iterations ($t\equiv1 \bmod 100$) & every $100$ steps, from $t=0$ \\
Records per run & $10^5$ & $201$ \\
\bottomrule
\end{tabular}
\end{table}

\subsection{Reproducibility and Determinism}
\label{app:sec:determinism}

Given the seed, the synthetic experiments are fully deterministic: the data are generated with a fixed seed, the initialization is fixed at zero, and gradient descent involves no further randomness. In the MNIST experiment, the network initialization is fixed by the seed, and gradient descent is full-batch. The \texttt{LinearSVC} solver used in the separability filter randomly permutes the coordinates of its dual coordinate descent and is not seeded, so the set of removed points may differ slightly between reruns; the counts reported above are those of the runs shown in \Cref{fig:Mnist Experiment}. GPU floating-point reductions are not forced to be deterministic, so reruns may differ at the level of double-precision rounding.

\subsection{Computing Infrastructure}
\label{app:sec:compute}

All experiments were run on a single NVIDIA DGX Station equipped with NVIDIA Tesla V100 GPUs with $32$ GB of HBM2 memory each. The synthetic experiments are implemented in \texttt{Python}~3.11 with \texttt{NumPy} and execute on the host CPU of the DGX Station; they use no GPU. Each run of $10^7$ iterations takes about $3$ minutes $50$ seconds (about $4.3\times10^4$ iterations per second) when metrics are evaluated every $100$ iterations. The run of \Cref{fig:accuracy_same_as_train_dim_2}, which evaluates the metrics at every iteration and writes every $100$th, takes $13$ minutes $40$ seconds. The MNIST experiment is implemented in \texttt{PyTorch} (network and gradients), \texttt{torchvision} (data loading) and \texttt{scikit-learn} (SVMs), and runs on a single V100 GPU. The full training set ($11\,605\times784$ double-precision values, about $73$ MB) and the network reside in GPU memory throughout training, well within the $32$ GB available.


\end{document}